\documentclass{article}

\PassOptionsToPackage{numbers, compress}{natbib}

    \usepackage[preprint]{neurips_2025}

\usepackage{wrapfig}
\usepackage{paralist}

\usepackage{algorithm}
\usepackage{algorithmic}
\usepackage[dvipsnames]{xcolor}
\usepackage{tabularx}
\usepackage{tabularray} 
\UseTblrLibrary{booktabs}
\usepackage{microtype}
\usepackage{graphicx}
\usepackage{subfigure}

\usepackage{adjustbox}
\usepackage{bbm}
\usepackage{multirow}
\usepackage{multicol}
\usepackage{booktabs}

\usepackage{pythonhighlight}
\usepackage[utf8]{inputenc} % allow utf-8 input
\usepackage[T1]{fontenc}    % use 8-bit T1 fonts
\usepackage{hyperref}       % hyperlinks
\usepackage{url}            % simple URL typesetting
\usepackage{booktabs}       % professional-quality tables
\usepackage{amsfonts}       % blackboard math symbols
\usepackage{nicefrac}       % compact symbols for 1/2, etc.
\usepackage{microtype}      % microtypography
\usepackage{xcolor}         % colors
\usepackage{amsmath,amsfonts,bm,amsthm,amssymb,bbm}

\def\eqref#1{equation~\ref{#1}}
\def\Eqref#1{Eq.~(\ref{#1})}
\def\1{\bm{1}}

\def\rvepsilon{{\mathbf{\epsilon}}}

\def\va{{\bm{a}}}

\def\vq{{\bm{q}}}

\def\vs{{\bm{s}}}

\def\vx{{\bm{x}}}
\def\vy{{\bm{y}}}

\def\mI{{\bm{I}}}

\DeclareMathAlphabet{\mathsfit}{\encodingdefault}{\sfdefault}{m}{sl}
\SetMathAlphabet{\mathsfit}{bold}{\encodingdefault}{\sfdefault}{bx}{n}

\def\gA{{\mathcal{A}}}
\def\gB{{\mathcal{B}}}

\def\gD{{\mathcal{D}}}

\def\gL{{\mathcal{L}}}

\def\gN{{\mathcal{N}}}

\def\gP{{\mathcal{P}}}

\def\gS{{\mathcal{S}}}
\def\gT{{\mathcal{T}}}
\def\gU{{\mathcal{U}}}
\def\gV{{\mathcal{V}}}

\def\sR{{\mathbb{R}}}

\def\gA{{\mathcal{A}}}
\def\gB{{\mathcal{B}}}

\def\gD{{\mathcal{D}}}

\def\gL{{\mathcal{L}}}

\def\gN{{\mathcal{N}}}

\def\gP{{\mathcal{P}}}

\def\gS{{\mathcal{S}}}
\def\gT{{\mathcal{T}}}
\def\gU{{\mathcal{U}}}
\def\gV{{\mathcal{V}}}

\newcommand{\E}{\mathbb{E}}

\newcommand{\KL}{D_{\mathrm{KL}}}

\newcommand{\policy}{\pi(\va\vert\vs)}

\newcommand{\cbr}[1]{\left\{#1\right\}}
\newcommand{\sbr}[1]{\left[#1\right]}
\newcommand{\pbr}[1]{\left(#1\right)}

\DeclareMathOperator*{\argmax}{arg\,max}
\DeclareMathOperator*{\argmin}{arg\,min}

\theoremstyle{plain}
\newtheorem{theorem}{Theorem}[section]
\newtheorem{proposition}[theorem]{Proposition}
\newtheorem{lemma}[theorem]{Lemma}

\theoremstyle{definition}
\newtheorem{definition}[theorem]{Definition}
\newtheorem{assumption}[theorem]{Assumption}
\theoremstyle{remark}
\newtheorem{remark}[theorem]{Remark}
\title{AlignIQL: Policy Alignment in Implicit Q-Learning through Constrained Optimization}

\author{%
    Longxiang He \\
  Center for Artificial Intelligence and Robotics
\\
  Tsinghua University\\
  \texttt{hlx22@mails.tsinghua.edu.cn} \\
  \And
  Li Shen \\
  Sun Yat-Sen University \\
  shenli6@mail.sysu.edu.cn \\
  \texttt{shenli6@mail.sysu.edu.cn} \\
  \And
  Xueqian Wang \\
  Center for Artificial Intelligence and Robotics\\
  Tsinghua University \\
  \texttt{wang.xq@sz.tsinghua.edu.cn} \\
}

\begin{document}

\maketitle

\begin{abstract}
Implicit Q-learning (IQL) serves as a strong baseline for offline RL, which never needs to evaluate actions outside of the dataset through quantile regression. However, it is unclear how to recover the implicit policy from the learned implicit Q-function, and weighted regression is theoretically justified as a policy extraction method in IQL. In this work, we reformulate the Implicit Policy Finding problem as an optimization problem. Based on this optimization problem, we provide insights into why IQL can use weighted regression for policy extraction and get a practical algorithm, AlignIQL, to solve this optimization problem, which inherits the advantages of decoupling actor from critic in IQL. Compared to IQL, our method retains its simplicity while addressing the implicit policy-finding problem. Experimental results demonstrate that enforcing policy alignment (AlignIQL) improves performance on challenging tasks such as AntMaze and enhances the robustness of the extracted policy.
\end{abstract}

\section{Introduction}
Offline Reinforcement Learning (RL), or Batch RL aims to seek an optimal policy without environmental interactions~\citep{fujimoto2019,levine2020}. This is compelling for having the potential to transform large-scale datasets into powerful decision-making tools and avoid costly and risky online environmental interactions, which offers significant application prospects in fields such as healthcare~\citep{nie2021learning,tseng2017deep} and autopilot~\citep{yurtsever2020survey,rhinehart2018deep}. Notwithstanding its promise, applying off-policy RL algorithms~\citep{lillicrap2015,fujimoto2018,haarnoja2018,haarnoja2018a} directly into the offline context presents challenges due to out-of-distribution actions that arise when evaluating the learned policy.\citep{fujimoto2019,levine2020}. 

Although a variety of methods based on constrained and conservative Q-learning have been proposed to address this problem, IQL~\citep{kostrikov2021} stands out among them since IQL avoids visiting out-of-distribution (OOD) actions and decouples the critic from the actor, which contributes to stability and hyperparameter robustness. For implicit policy extraction, IQL extracts policy through advantage-weighted regression (AWR)~\citep{nair2020,peng2019,peters2010}. However, the general form of extracted policy is $\pi(\va|\vs)\propto\mu(\va|\vs)w(\vs,\va)$, where $\mu(\va|\vs)$ is the behavior policy. The AWR's weight used by IQL is obtained from the constrained policy search, which does not guarantee that it is the policy the learned IQL’s value function is actually evaluating~\citep{hansen-estruch2023}.

To solve this problem, IDQL~\citep{hansen-estruch2023} reinterprets IQL as an actor-critic method and derives the implicit optimal policy weights. Nevertheless, this optimal weight hinges on the assumption that the optimal value function can be learned under certain critic loss functions. It remains unclear whether using AWR to extract policies for IQL is feasible, and how to extract policies from arbitrary critic loss function, not just the expectile loss. Recently, this issue has become more important since 1) many recent offline RL~\citep{chen_score_2023} and safe RL methods~\citep{zheng2023feasibility} use IQL to learn the Q-function; 2) IQL's performance is significantly affected by the choice of a policy extraction algorithm~\citep{tarasov2024role}. Addressing these issues can lead to a better understanding of IQL-style methods' bottlenecks, thereby promoting the development of offline RL.  

In this paper, we address the above issues by formulating the implicit policy-finding problem as an optimization problem, where the objective function is a generalized form of behavior regularizers and the constraint is policy alignment. Policy alignment ensures the extracted policy is the policy implied in the Q-function. By solving this optimization problem, we can get a closed-form solution, which can be expressed by imposing weight on the behavior policy. The weight consists of a value function, an action-value function, and multipliers, indicating that using AWR to extract IQL policies is feasible only when the certain multiplier is less than $0$, and this conclusion can be generalized to any value function loss. Furthermore, our work also explains how the implicit policy in IQL-style methods addresses OOD actions from the perspective of behavior regularizers.

Based on the optimization problem, we present two algorithms, AlignIQL-hard and AlignIQL. Both inherit the characteristics of IQL, $i.e.$ the decoupling of actor and critic training. AlignIQL-hard can theoretically achieve a globally optimal solution, but it is more vulnerable to hyperparameter choices than AlignIQL. AlignIQL relaxes the policy alignment constraint and performs better in complex tasks like sparse rewards tasks, but it does not guarantee convergence to the global optimum.

Recently, Diffusion models~\citep{sohl-dickstein2015,ho2020,song2019} have been widely used in Offline RL, since behavior policy is often complex and potentially multimodal, the unimodal Gaussian policy used in IQL is unlikely to accurately approximate the complex behavior policy~\citep{wang2022b,hansen-estruch2023,chen2022,he2024diffcps}, which in turn affects the implicit policy extraction. Our method can also be easily combined with diffusion models. We just need to resample the actions generated by the diffusion-parameterized behavior model according to the weights $w(\vs,\va)$ of our method. 
% We, therefore, use the diffusion model~\citep{sohl-dickstein2015,ho2020,song2019} to model the behavior policy.
% Noting that recent works have shown that diffusion models can address the issue of limited expressivity problem in offline RL \ls{rewrite this sentence.}
We evaluate the effectiveness of our method on D4RL datasets, image-based control tasks, and robustness benchmarks. Experimental results show that implementing policy alignment (AlignIQL) leads to improved performance on challenging tasks such as the AntMaze tasks and enhances the robustness of the extracted policy.
 
To summarize, our main contributions are as follows:
\begin{itemize}
    \item We propose the policy-finding problem, where the policy alignment term is added as a constraint. By solving this problem, we provide insights into why and when IQL can use weighted regression for policy extraction, and in turn, make it better to understand the bottlenecks of the IQL-style algorithms.
    \item We demonstrate that there is no price to achieving policy alignment in IQL-style methods, all we need is to modify the importance weights of the extracted policy. These results can be generalized to any generalized value loss function, which greatly extends the theoretical results of IDQL. 
    \item We propose AlignIQL, a novel IQL policy extraction method that achieves policy alignment. Extensive experiments demonstrate that enforcing policy alignment enhances the robustness of the extracted policy and improves performance on challenging tasks.

\end{itemize}

\section{Related Works}
 \textbf{Offline RL}. Offline RL algorithms need to avoid OOD actions. Previous methods to mitigate this issue under the model-free offline RL setting generally fall into three categories: 1) value function-based approaches, which implement pessimistic value estimation by assigning low values to out-of-distribution actions~\citep{kumar2020,fujimoto2019}, or implicit TD backups~\citep{kostrikov2021,ma2021offline} to avoid the use of out-of-distribution actions 2) sequential modeling approaches, which casts offline RL as a sequence generation task with return guidance~\citep{chen2021,janner2022,liang2023,ajay2022}, and 3) constrained policy search (CPS) approaches, which regularizes the discrepancy between the learned policy and behavior policy~\citep{peters2010,peng2019,nair2020}. 
 
 \textbf{Implicit Q-learning}. Implicit Q-learning~\citep{kostrikov2021} has attracted interest due to its stable training and simplicity. Many offline RL methods~\citep{chen_score_2023,zheng2023feasibility,hansen-estruch2023}  use IQL-style expectile regression to learn Q-function and realize the advantage of decoupling the training of actor and critic. While IQL achieves superior performance, several issues remain unsolved. SQL~\citep{xuOfflineRLNo2023} reinterprets IQL in the Implicit Value Regularization (IVR) framework and provides insights about why in practice a large  $\tau$ may give a worse result in IQL. However, there is another important open question about IQL, that is, what policy the learned value function is evaluating. IDQL~\citep{hansen-estruch2023} solves this by reinterpreting the IQL as an actor-critic method and getting the corresponding implicit policy for the (generalized) IQL loss function. However, the corresponding implicit policy in IDQL only holds for optimal value function under certain critic loss functions.  

The closest work to ours is IDQL~\citep{hansen-estruch2023}, which derives the implicit policy for optimal value function under different critic loss functions. Our method is related, but features with AlignIQL can be applied to arbitrary sub-optimal value functions and arbitrary critic loss functions. More importantly, our method explains when and why IQL can use AWR for policy extraction while providing theoretical insights for IQL and other RL paradigms that use Q-values to guide sampling. 

% \ls{add an additional paragraph to highlight the difference between this work and the most related works.}
% \ls{The most related work to ours is XXX []. The main difference lies in XXX-fold: (i) (ii)(iii)..}

\section{Background}
\textbf{Offline RL.} Consider a Markov decision process (MDP): $M = \{\gS, \gA, P, R, \gamma, d_0\}$, with state space $S$, action space $\gA$, environment dynamics $\gP(\vs' \vert \vs, \va): \gS \times  \gA \times \gS \rightarrow [0,1]$, reward function $R: S \times \gA \rightarrow \sR$, discount factor $\gamma \in [0, 1)$, policy  $\pi(\va|\vs):\gS\times\gA\rightarrow [0,1]$, and initial state distribution $d_0$. The action-value or Q-value of policy $\pi$ is defined as $Q^\pi(\vs_t, \va_t) = \E_{\va_{t+1}, \va_{t+2}, ... \sim \pi}\sbr{\sum_{j=0}^{\infty}\gamma^j r(\vs_{t+j}, \va_{t+j})}$. The value function of policy $\pi$ is defined as $V^\pi(\vs)=\int_\mathcal{A}Q^\pi(\vs, \va)\pi(\va|\vs) d\va$. The goal of RL  is  to get a policy to maximize  the cumulative discounted reward 
$J(\pi) = \int_\mathcal{S} d_0(\vs) V^\pi(\vs) d\vs$.  $d^\pi(\vs)=\sum_{t=0}^\infty\gamma^t p_\pi(\vs_t =\vs)$ is the state visitation distribution induced by policy $\pi$~\citep{sutton2018,peng2019}, and $p_\pi(\vs_t =\vs)$ is the likelihood of the policy being in state $\vs$ after following $\pi$ for $t$ timesteps. In offline setting~\citep{fujimoto2019}, environmental interaction is not allowed, and a static dataset $\gD \triangleq \cbr{(\gS, \gA, R, \gS', \text{done})}$ is used to learn a policy.

\textbf{Advantage Weighted Regression (AWR)}. Prior works~\citep{peters2010,peng2019} formulate offline RL as a constrained policy search (CPS) problem with the following form:
\begin{align}\label{origin_cps}
        \pi^{*} = \argmax_\pi \ &J(\pi) = \argmax_\pi \int_\mathcal{S} d_0(\vs) \int_\mathcal{A}\pi(\va|\vs)Q^\pi(\vs, \va) d\va d\vs  \nonumber\\
         s.t.  \quad & \KL(\mu(\cdot\vert\vs)\|\pi(\cdot\vert\vs)) \leq \rvepsilon,\quad \forall \vs\\  
         & \int_\va \pi(\va\vert\vs)d\va =1,\quad \forall \vs, \nonumber
\end{align}
Previous works~\citep{peters2010,peng2019,nair2020} solve \Eqref{origin_cps} through KKT conditions and get the optimal policy $\pi^*$ as:
\begin{align}
    \pi^*(\va\vert\vs) &= \frac{1}{Z(\vs)} \ \mu(\va\vert\vs) \ \mathrm{exp}\left(\alpha Q_\theta(\vs, \va) \right),
\label{Eq:pi_optimal}
\end{align}
where $Z(\vs)$ is the partition function, $\alpha\geq0$ is a Lagrange multiplier, and $Q_\theta$ is a learned Q-function of the current policy $\pi$. Intuitively we can use \Eqref{Eq:pi_optimal} to optimize policy $\pi$. However, the behavior policy may be very diverse and hard to model. To avoid modeling the behavior policy, prior works~\citep{peng2019,wang2020,chen2020} optimize $\pi^{*}$ through a parameterized policy $\pi_\phi$, known as AWR:
\begin{align}
    & \mathop{\mathrm{arg \ min}}_{\phi} \mathbb{E}_{\vs \sim \mathcal{D}^\mu} \left[ \KL \left(\pi^*(\cdot  | \vs) \middle|\middle| \pi_\phi(\cdot  | \vs)\right) \right]  \label{Eq:wr} \\
    = & \mathop{\mathrm{arg \ max}}_{\phi} \mathbb{E}_{(\vs, \va) \sim \mathcal{D}^\mu} \left[ \frac{1}{Z(\vs)} \mathrm{log} \ \pi_\phi(\va | \vs) \ \mathrm{exp}\left(\alpha Q_\theta(\vs, \va) \right) \right].\nonumber
\end{align}
where  $\mathrm{exp}(\alpha Q_\theta(\vs, \va))$ being the regression weights. 
% However, AWR requires the exact probability density of policy, which restricts the use of generative models like diffusion models. In this paper, we directly utilize the diffusion-based policy to address Eq.~(\ref{origin_cps}). Therefore, our method not only avoids the need for explicit probability densities in AWR but also solves the limited policy expressivity problem.

\textbf{Implicit Q-learning (IQL).} To avoid OOD actions in offline RL, IQL~\citep{kostrikov2021} uses the state conditional upper expectile of action-value function $Q(\vs,\va)$ to estimate the value function $V(\vs)$, which avoid directly querying a Q-function with unseen action.  For a parameterized critic $Q_\theta(\vs,\va)$, target critic $Q_{\hat \theta}(\vs,\va)$, and value network $V_{\psi}(\vs)$ the value objective is learned by 
\begin{equation}
    \begin{aligned}    
        \label{eqn:fit_v_expectiles}
        &\gL_V(\psi) = \E_{(\vs,\va)~\sim \gD}[L_2^\tau(Q_{\hat{\theta}}(\vs,\va) - V_\psi(\vs))] \\ 
        &\text{where} \quad L_2^\tau(u) = |\tau-\mathbbm{1}(u < 0)|u^2,
\end{aligned}
\end{equation}
where $\mathbbm{1}$ is the indicator function. 
Then, the Q-function is learned by minimizing the MSE loss
\begin{align}
    \label{eqn:fit_q}
    \mathcal{L}_Q(\theta) &= \mathbb{E}_{(\vs,\va,\vs')~\sim \mathcal{D}}[(r(\vs,\va) + \gamma V_\psi(\vs') - Q_{\theta}(\vs,\va))^2].
\end{align}
Note that, in IQL, the policy is not explicitly represented, it is implicit in the learned value function. For policy extraction, IQL uses \Eqref{Eq:wr} in AWR \citep{peters2010,peng2019,nair2020}, which trains the policy through weighted regression by minimizing $\mathcal{L}_\pi(\phi)$
 \begin{align}
 \label{AWR}
  \mathbb{E}_{(\vs,\va)\sim \mathcal{D}}[-\exp(\alpha(Q_{\hat \theta}(\vs,\va) - V_{\psi}(\vs)))\log \pi_{\phi}(\va|\vs)].
\end{align}

However, it is still unclear whether AWR can be used to extract policies for IQL.
 Answering this question can help us better understand the bottlenecks of IQL-style methods.

% \ls{the importance of the feasibility should be clearly stated.}
\section{Implicit Policy-finding Problem}

% \ls{it is not clear "what is the implicit policy finding problem" and "how to model this problem and why?"}
Before presenting our method, we formally introduce the definition of policy alignment and the formulation of the Implicit Policy-Finding Problem. We begin with Definition~\ref{def1}, which characterizes the policy implied by the value function. Policy alignment is considered achieved if the learned value function and the extracted policy satisfy the conditions in Definition~\ref{def1}.

\begin{definition}
\label{def1}
    \textbf{Policy Alignment:} We refer to a policy as one implied by the value function $Q(\vs,\va), V(\vs)$, when 
    \begin{equation}
    \label{ignored}
    Q(\vs,\va)-r(\vs,\va)-\gamma\E_{\vs'\sim p(\vs' | \vs,\va),\va'\sim\pi(\va'|\vs')}\sbr{Q(\vs',\va')}=0.
\end{equation}
\begin{equation}
\label{align_cs}
    \E_{\va\sim\pi(\va|\vs)}\sbr{Q(\vs,\va)}=V(\vs), 
\end{equation}
\end{definition}
Definition~\ref{def1} is derived from IDQL~\citep{hansen-estruch2023} and the conventional definition of the value function in actor-critic methods. 
% Definition~\ref{def1} shows if we can constrain the policy to satisfy Definition~\ref{def1}, then we can ensure policy alignment. 
Note that in IQL, the $Q$-function is updated by minimizing~\Eqref{eqn:fit_q}. This implies that if \Eqref{align_cs} holds, \Eqref{ignored} can be derived by substituting \Eqref{align_cs} into \Eqref{eqn:fit_q} and then setting the gradient with respect to $Q_\theta$ to zero. So in the following sections, we eliminate \Eqref{ignored} and use \Eqref{align_cs} as the policy alignment constraint. 

It is known that the offline RL problem can be solved by the constrained policy search (CPS) problem (aka AWR)~\citep{nair2020,peng2019,peters2010}, where a policy is sought to maximize cumulative rewards under the constraint of policy divergence from the behavior policy. Inspired by CPS, we formulate the \textit{implicit policy-finding problem} (IPF) as a constrained optimization problem, where a policy is sought to minimize  policy divergence from the behavior policy under policy alignment
\begin{equation}
\label{align_p}
\tag{IPF}
    \begin{aligned}
         \min_{\pi} \quad  &\E_{\vs\sim d^\pi(\vs),\va\sim\policy}\sbr{f\pbr{\frac{\policy}{\mu(\va\vert\vs)}}}\\
        & s.t. \quad \policy\geq 0, \quad \forall \vs,\forall \va \\
        & \quad \int_\va\policy d\va = 1, \quad \forall \vs \\ 
        & \quad \E_{\va\sim\pi(\va|\vs)}\sbr{Q(\vs,\va)}-V(\vs)=0,\quad \forall \vs,\\
        % & \quad Q(\vs,\va)-r(\vs,\va)-\gamma\E_{\vs'\sim p(\vs' | \vs,\va),\va'\sim\policy}\sbr{Q(\vs',\va')}=0\\
    \end{aligned}
\end{equation}
where $V(\vs), Q(\vs,\va)$ is the learned value function, which does not have to be the optimal value function. $f(\cdot)$ is a regularization function which aims to avoid out-of-distribution actions. The third constraint ensures that the extracted policy is the policy implied in $Q, V$. 

 % Although constraint $3$ in problem~\ref{align_p} is derived from IDQL~\citep{hansen-estruch2023}, our method differs from it as Problem~\ref{align_p} applies to any learned $Q$-function rather than just the optimal value function. 
 Here we briefly describe the characteristics of the solution to problem~\ref{align_p}. In problem~\ref{align_p}, when the feasible set includes multiple policies ($i.e.$ 
 multiple implicit policies satisfy Definition~\ref{def1}), problem~\ref{align_p} aims to find an optimal implicit policy that deviates least from the behavior policy while satisfying the requirements of policy alignment. In other cases, when the feasible set has a unique policy, problem~\ref{align_p} will return the unique policy as the optimal implicit policy. The above analysis shows that we can model the implicit policy-finding problem in IQL as problem~\ref{align_p}. 
 % \ls{the goal of this part is not clear.}

\begin{assumption}
    Assume $\pi(\va|\vs)>0\Longrightarrow\mu(\va|\vs)>0$ so that $\frac{\pi(\va|\vs)}{\mu(\va|\vs)}$ is well-defined.~\citep{xuOfflineRLNo2023}
\end{assumption}

\begin{assumption}
\label{as2}
    Assume that $f(x)$ is differentiable on $(0,\infty)$ and that $h_f(x)=xf(x)$ is strictly convex and $f(1)=0$.~\citep{xuOfflineRLNo2023}
\end{assumption}
\begin{remark}
    Under the above assumptions, problem~\ref{align_p} is a convex optimization problem and assumption~\ref{as2} makes the regularization term positive due to Jensen's inequality as $\E_\mu[\frac{\pi}{\mu}f(\frac{\pi}{\mu})]\geq1, f(1)=0$~\citep{xuOfflineRLNo2023}. Slater's conditions hold since the first and second constraints define a probability simplex, and the third constraint defines a hyperplane in the tabular setting. The intersection of these convex sets is nonempty if the optimal policy exists, $i.e.$ the optimal policy is not a uniform distribution. The analysis described above shows that this convex optimization problem is feasible and Slater's conditions are satisfied. 
\end{remark}

% In section~\ref{experiment}, we will show this alignment will be helpful in the Offline-to-online (O2O) fine-tuning and sparse rewards settings.

\section{Optimization}
% In this section, we first explain when and why AWR can be used for policy extraction in IQL by solving problem~\ref{align_p} and get AlignIQL-hard. In theory, AlignIQL-hard can achieve global optimal but suffers from complex training. To solves this, we relax  problem~\ref{align_p} and get a closed solution from the relaxed problem, $i.e.$ AlignIQL. AlignIQL avoids the training complexity of AlignIQL-hard while also guaranteeing local convergence to the optimal solution of problem~\ref{align_p}  through soft constraints. All proof can be found in Appendix~\ref{all_prof}.
In this section, we first solve Problem~\ref{align_p} to explain when and why AWR can be used for policy extraction in IQL, leading to AlignIQL-hard. Theoretically, AlignIQL-hard can achieve global optimality but faces a complex training process. To address this issue, we relax Problem~\ref{align_p} and derive a closed-form solution from the relaxed problem, namely AlignIQL. AlignIQL avoids the training complexity of AlignIQL-hard while ensuring that the optimal solution of Problem~\ref{align_p} is also a local optimum of AlignIQL. All proofs can be found in Appendix~\ref{all_prof}.
% \ls{not clear why we introduce two types of methods to solve this problem.}

\subsection{Hard Constraint Solving}
\label{aligniql-hard}
We first consider directly solving \ref{align_p} with KKT conditions (See proof in Appendix~\ref{theorem1}) and get the following theorems. 
\begin{theorem}
\label{optimal_policy}
    For  problem~\ref{align_p}, the optimal policy $\pi^*$ and its optimal Lagrange multipliers satisfy the following optimality condition for all states and actions:
    \begin{equation}
    \label{pi_star}
    \pi^\star(\va\vert\vs) = \mu(\va\vert\vs)\max\cbr{g_f\pbr{-\alpha^*(\vs)-\beta^*(\vs) Q(\vs,\va)},0}.
\end{equation}
\begin{equation}
    \label{loss_1}
    \E_{\va\sim\mu}\sbr{\max\cbr{g_f(-\alpha^*(\vs)-\beta^*(\vs) Q(\vs,\va)),0}}=1,
\end{equation}
\begin{equation}
\begin{aligned}
    &\E_{\va\sim\mu(\va|\vs)}[Q(\vs,\va)\max\cbr{g_f(-\alpha^*(\vs)-\beta^*(\vs) Q(\vs,\va)),0}-V(\vs)]=0,
\end{aligned}
    \label{loss_2}
\end{equation}
where $\alpha^*,\beta^*$ is the Lagrange multiplier, $g_f$ is the inverse function of h$'_f(x)$.
\end{theorem}
% We can also follow ~\citet{peters2010,peng2019,nair2020} to train our policy $\pi_\phi$ through 
% \begin{equation}
% \label{aligniqlhard-loss}
%     \begin{aligned}
%     & \mathop{\mathrm{arg \ min}}_{\phi} \mathbb{E}_{\vs \sim \mathcal{D}^\mu} \left[ \KL \left(\pi^*(\cdot  | \vs) \middle|\middle| \pi_\phi(\cdot  | \vs)\right) \right]  \\
%     \approx & \mathbb{E}\sbr{-\max\cbr{g_f\pbr{-\alpha^*(\vs)-\beta^*(\vs) Q(\vs,\va)},0}\log \pi_{\phi}(\va|\vs)}.   
% \end{aligned}
% \end{equation}
% However, loss function~\Eqref{aligniqlhard-loss} needs the exact policy density, which may limit the usage of diffusion models or other generative models.

\textbf{Connection to AWR}:  Note that $\alpha^*$ is a normalization term, it does not affect the action generated by the policy. 
    % Since $h_f'$ is a strictly increasing function, its inverse function $g_f$ exists and is also a strictly increasing function.  
    Let $f(x)=\log x$, then $g_f(x)=\exp{(x-1))}>0$, we can get $\pi^*(\va|\vs)\propto \mu(\va|\vs)\exp{(-\beta^* Q(\vs,\va))}$ In most environments (especially MuJoCo tasks), $\beta^*$ we learned through the neural network is negative. We can rewrite $-\beta^*$ with a fixed $\beta\in(0,\infty]$, $i.e.$ $\pi^*(\va|\vs)\propto \mu(\va|\vs)\exp{(\beta Q(\vs,\va))}$, which is exactly what optimal policy obtained by AWR. This explains why IQL can learn implicit policy with weighted regression and shows implicit policy further avoids the OOD actions through the regularization function $f$, which gives a deeper understanding of how IQL-style methods handle the distribution shift. This also addresses the issue in IDQL, as they find that simply selecting the action with the highest Q-value at evaluation time usually leads to better performance since the policy for some tasks is expressed as $\pi^*(\va|\vs)\propto \mu(\va|\vs)\exp{(\beta Q(\vs,\va))}$.

 Previous works~\citep{hansen-estruch2023,chen2022} often use the increasing function of $Q(\vs,\va)$ as a weight. However, according to Theorem~\ref{optimal_policy}, when $\beta^*(\vs)\geq 0 $, we need to be more conservative, that is, we should choose actions with lower $Q(\vs,\va)$. To calculate the weights, we need to solve the closed-form solution of \Eqref{loss_1}, \Eqref{loss_2}, which is usually intractable. However, we can use the parameterized neural network to approximate it.
\begin{lemma}
\label{loss_lemma}
     Following EQL~\citep{xuOfflineRLNo2023}, let $f(x)=\log x$, then $g_f(x)=\exp{(x-1))}>0$. We can approximate $\alpha^*(\vs)$, $\beta^*(\vs) $ through neural network  with the following loss function:
     \begin{equation}
         \label{loss_multiplier}
         \begin{aligned}
                      &\max_{\alpha, \beta} \gL_M=-\E_{\va\sim\mu}\sbr{\exp{\pbr{-\alpha(\vs)-\beta(\vs) Q(\vs,\va)-1}}}-\alpha(\vs)-\beta (\vs) V(\vs),
         \end{aligned}
     \end{equation}
\end{lemma}
\begin{proof}
Then Lemma~\ref{loss_lemma} can be get through setting the  gradient of \Eqref{loss_multiplier} to $0$ with respect to $\alpha,\beta$, which is \Eqref{loss_1}, \Eqref{loss_2}  respectively. 
\end{proof}
\begin{remark}
Now we can obtain $\alpha^*,\beta^*$ by iteratively updating $\alpha,\beta$ following \Eqref{loss_multiplier}. 
\end{remark}

Based on Theorem~\ref{optimal_policy} and Lemma~\ref{loss_lemma}, we can get AlignIQL-hard, where hard means we rigidly constrain the policy to satisfy policy alignment. AlignIQL-hard shows when multiplier $\beta(\vs)<0$, we can use AWR for extracting the implicit policy in IQL. However, for strict policy alignment, AlignIQL-hard needs to train an additional two multiplier networks, which increases the training costs and compound errors. Moreover, the exponential term in \Eqref{loss_multiplier} makes the unstable training. In the remainder of this section, we introduce a simple and effective method AlignIQL to solve problem~\ref{align_p}.

% To solve this problem, we introduce AlignIQL-hard, which approximates Lagrange multiplier $\beta^*(\vs),\alpha^*(\vs)$ with the neural network. Lemma~\ref{loss_lemma} shows how we update it.
% \ls{not clear why we call it AlignIQL-hard. what is the meaning of "hard"? We should make the introduction of this method more clear. }
% \ls{We should make clear comments on the AlignIQL-hard methods, e.g., it requires to training 2 NNs. }
% \ls{the main algorithm is AlignIQL, why do we introduce AlignIQL-hard?} 
% \ls{the total section should be reconsidered.}
% \longx{solved}

\subsection{Soft Constraint Solving}
\label{soft_solve}
In this section, we introduce AlignIQL to solve the alignment problem of IQL. Firstly, we introduce the soft constraint form of problem~\ref{align_p}. Given $\eta>0$, IPF-Soft is defined as 
\begin{equation}
\label{align_soft}
\tag{IPF-Soft}
    \begin{aligned}
         \min_{\pi,V(\vs)} \E_{\substack{\vs\sim d^\pi(\vs)\\ \va\sim\policy}}&\sbr{f\pbr{\frac{\policy}{\mu(\va\vert\vs)}}+\eta\pbr{Q(\vs,\va)-V(\vs)}^2}\\
        &s.t. \quad  \policy\geq 0, \quad \forall \vs,\forall \va \\
        & \quad  \int_\va\policy d\va = 1, \quad \forall \vs. \\ 
    \end{aligned}
\end{equation}
\begin{remark}
 Note that we relax problem~\ref{align_p} by adding penalty term $\mathbb{E}_{\va\sim\pi(\va|\vs)}[\eta\pbr{Q(\vs,\va)-V(\vs)}^2]$ rather than $\eta(\mathbb{E}_{\va\sim\pi(\va|\vs)}[Q(\vs,\va)]-V(\vs))^2$. The latter relaxation formulation is equivalent to the quadratic penalty method, whose convergence relies on the penalty parameter $\eta$ 
 approaching positive infinity which leads to an ill-conditioned Hessian matrix for the quadratic penalty function~\citep{nocedal1999numerical}. Our penalty term can avoid this issue since the optimal solution of $\mathbb{E}_{\va\sim\pi(\va|\vs)}[\eta\pbr{Q(\vs,\va)-V(\vs)}^2]$ satisfies \Eqref{align_cs} (setting the gradient to $0$ with respect to $V$), which shows that our penalty term can implicitly recover policy alignment constraint~\Eqref{align_cs}.

 % Our penalty function is inspired by Theorem 4.1 from IDQL~\citep{hansen-estruch2023}, which can avoid the aforementioned issues. In fact, we can  This is the main reason why AlignIQL relax IPF with $\eta\pbr{Q(\vs,\va)-V(\vs)}^2$.
\end{remark}
    We refer to the above problem as problem~\ref{align_soft}, since the policy alignment is not rigidly held. Then we solve problem~\ref{align_soft} by KKT conditions and get the optimal policy $\pi^\star(\va\vert\vs)$:
\begin{equation}
    \label{pi_s_star_1}
     \mu(\va\vert\vs)\max\cbr{g_f\pbr{-\alpha(\vs)-\eta\pbr{Q(\vs,\va)-V(\vs)}^2},0}.
\end{equation}

\begin{theorem}
\label{soft-optimal}
    Suppose that $f(x)=\log x$, then the optimal policy of problem~\ref{align_soft} satisfies
\begin{equation}
    \label{pi_soft}
    \pi^\star(\va\vert\vs) \propto \mu(\va\vert\vs)\exp{\cbr{-\eta\pbr{Q(\vs,\va)-V(\vs)}^2}}.
\end{equation}
If the exact policy density is known, we can also follow ~\citet{peters2010,peng2019,nair2020} to train our policy $\pi_\phi$ through 
\begin{equation}
\label{aligniql-loss}
    \begin{aligned}
    & \mathop{\mathrm{arg \ min}}_{\phi} \mathbb{E}_{\vs \sim \mathcal{D}^\mu} \left[ \KL \left(\pi^*(\cdot  | \vs) \middle|\middle| \pi_\phi(\cdot  | \vs)\right) \right]  \\
    \approx & \mathbb{E}\sbr{-\exp{\pbr{-\eta\pbr{Q(\vs,\va)-V(\vs)}^2}}\log \pi_{\phi}(\va|\vs)}.    
\end{aligned}
\end{equation}

\end{theorem}

\textbf{Compared to AWR:} For $\eta > 0$, \Eqref{pi_soft} favors actions that minimize $(Q(\vs,\va) - V(\vs))^2$. This contrasts with AWR, which prefers actions associated with higher $Q(\vs,\va)$ values. The key distinction arises from AlignIQL’s objective of balancing behavior cloning with policy alignment, whereas AWR seeks to balance behavior cloning with critic exploitation. However, this exploitation overlooks the confidence associated with each $(\vs, \va)$ pair. For instance, although AWR performs well in most scenarios, the estimation of $Q$ may be unreliable in challenging or corrupted tasks. In such cases, AWR continues to assign large weights to high $Q$ values, while AlignIQL instead emphasizes policy alignment by assigning greater weight to $Q$ values that are closer to $V$. We will empirically validate this in Section~\ref{experiment_policy_align}.

\textbf{The Optimality of AlignIQL's Weight:}  As a policy extraction method, the optimality of AlignIQL primarily stems from the optimality of the value function. As observed by~\citet{tarasov2024role}, the value function is typically learned more accurately than the policy itself. Furthermore, AlignIQL can recover the optimal policy under certain conditions. For example, in IQL, the expectile loss~\Eqref{eqn:fit_v_expectiles} approximates the maximum of $Q_{\hat{\theta}}(\vs,\va)$ when $\tau \approx 1$. In this case, we can roughly interpret $V(\vs) = \argmax_{\va \sim \gD} Q(\vs,\va)$. According to \Eqref{pi_soft}, the action $\hat{\va} = \argmax_\va Q(\vs, \va)$ receives a weight of 1, while all other actions are weighted by $\exp\cbr{-\eta (Q(\vs, \va) - V(\vs))^2}$. For fixed $\eta$, these weights are strictly smaller than that of the maximal action. As a result, \Eqref{pi_soft} approximately recovers the optimal policy $\pi(\va|\vs) = \argmax_{\va \sim \gD} Q(\vs,\va)$.

Finally, we show the connection between the solution of problem~\ref{align_p} and problem~\ref{align_soft} through the following Proposition~\ref{connect}.
\begin{proposition}
\label{connect}

Suppose that $\pi^*(\va\vert\vs)$ is a global solution to the convex optimization problem~\ref{align_p}, 
 with its corresponding value function (denoted as $V^*(\vs)$). Then there exists a $\eta$ such that $\pi^*,V^*(\vs)$ is a local minimizer of problem~\ref{align_soft}. (See proof in Appendix~\ref{proposition}.)
%  The global optimal solution to the problem~\ref{align_p} is $p^*$. Then $\pi^*,V^*(\vs)$ is a strict local minimizer of problem~\ref{align_soft} for all  
%  $\eta \leq (k^*-p^*)/h^*,$ where 
%     $h^*=\E_{\substack{\vs\sim d^\pi(\vs)\\\va\sim\pi^*(\va|\vs)}}\sbr{\pbr{Q(\vs,\va)-V^*(\vs)}^2},$
%     $$
%     k^*=\min_{\substack{\pi \\ s.t. \pi\in\gT}}  \E_{\substack{\vs\sim d^\pi(\vs)\\ \va\sim\policy}}\sbr{f\pbr{\frac{\policy}{\mu(\va\vert\vs)}}+\eta\pbr{Q(\vs,\va)-V(\vs)}^2},$$
% $$\gT=\cbr{\pi\vert\pi(\va\vert\vs)\geq0,\int_\va\pi(\va\vert\vs)d\va=1,V(\vs)=\E_{\va\sim\pi(\va|\vs)}\sbr{Q(\vs,\va)},\pi\in\mathring{U}(\pi^*,\sigma)}.$$

\end{proposition}
\begin{remark}
    Proposition~\ref{connect} indicates that we can obtain the solution to problem~\ref{align_p} by solving problem~\Eqref{align_soft}. Because KKT conditions are the first-order necessary for a solution in nonlinear programming to be optimal, the solution to problem~\ref{align_p} can be written in
    the form of \Eqref{pi_soft}. This implies that if we train $Q(\vs,\va)$ and $V(\vs)$ using IQL and $\eta$ satisfies Proposition~\ref{connect}, we can extract the implicit policy from the value function using \Eqref{pi_soft}. 
\end{remark}

\textbf{Two ways to use AlignIQL:} There are two ways to utilize our method in offline RL (corresponding to Algorithm~\ref{alg:iql} and Algorithm~\ref{alg:extraction}, Suppose that $f(x)=\log x$). 
\begin{itemize}
    \item \textbf{Gaussian-based implementation:} We employ \Eqref{aligniql-loss} to train the policy, which requires the exact probability density of the current policy (Algorithm~\ref{alg:iql}). Notably, accurately modeling \Eqref{aligniql-loss} necessitates that the policy $\pi_\phi$ possess strong distribution modeling capabilities, as the squared term increases the complexity of the learned distribution. This observation motivates our adoption of a diffusion-based implementation of AlignIQL.

    \item \textbf{Diffusion-based implementation:} We first use the learned diffusion-based behavior model $\mu_\phi(\va|\vs)$ to generate $N$ action samples. These actions are then evaluated using weights from~\Eqref{pi_soft} or \Eqref{pi_star} (Algorithm~\ref{alg:extraction}). In this setting, the hyperparameter $N$ has a greater influence on performance than $\eta$, as a higher $N$ is more likely to find the “lucky” action that satisfies $\hat{\va} = \argmax_\va Q(\vs, \va)$.  
\end{itemize}

 Note that in both AlignIQL-hard and AlignIQL, we do not impose a limit on the loss function of the $Q-V$, which means that our conclusion can be generalized to the arbitrary critic loss function and the arbitrary sub-optimal value function. To summarize, both the AlignIQL-hard and AlignIQL are "IQL-style" algorithms, which means the training of actor and critic are decoupled and the critic is learned by expectile regression. The difference between AlignIQL-hard and AlignIQL lies in the calculation of weights and the necessity of training multiplier networks. 
% \ls{we should also give some comments on the designed algorithm. }

\section{Experiments}
\label{experiment}
In this section, we empirically evaluate the advantages of policy alignment and the effectiveness of AlignIQL through D4RL AntMaze tasks, noise-corrupted data, and vision-based experiments.

\subsection{D4RL AntMaze Results}
\label{experiment_policy_align}
% \ls{We should also discuss whether the experimental results align with our theoretical findings.}
To validate the advantages of policy alignment, we first compare the performance of D-AlignIQL (Diffusion-based AlignIQL) against other diffusion-based baselines on the AntMaze tasks. The AntMaze tasks~\citep{fu2020}, which involve controlling an ant robot to navigate through a maze, are particularly challenging due to their increased demand for trajectory stitching. We choose the AntMaze tasks because, as noted in Section~\ref{aligniql-hard}, $\beta(\vs)$ is generally negative in MuJoCo tasks. Consequently, AWR alone is sufficient to achieve policy alignment, making AlignIQL redundant. We use D-AlignIQL instead of the Gaussian-based AlignIQL because implementing policy alignment requires strong policy modeling capacity, as discussed in Section~\ref{soft_solve}. Implementation details and additional empirical results are provided in Appendix~\ref{details}. We also include the full results of Diffusion-based AlignIQL and Gaussian-based AlignIQL in Appendix~\ref{add_exps}.

\textbf{Baselines:} We include DiffusionQL~\citep{wang2022b}, QGPO~\cite{lu2023}, EDP~\cite{kang_efficient_nodate}, SRPO~\citep{chen_score_2023}, DTQL~\cite{chen2024diffusion}, and SfBC~\citep{chen2022} as diffusion-based baselines due to their strong performance in offline RL. Notably, SfBC is a diffusion+AWR method that first trains a diffusion-based policy and then selects actions based on $Q$ values.

As shown in Table~\ref{tbl:rl_results_maze}, D-AlignIQL achieves the highest average performance and ranks among the top two in 5 out of 6 tasks, matching or surpassing other diffusion-based offline RL methods. Notably, SfBC corresponds to diffusion+AWR, and IDQL is also motivated by policy alignment. D-AlignIQL outperforms both on 5 out of 6 tasks, indicating the superiority of our policy alignment weighting. Compared to other diffusion-based methods, D-AlignIQL also demonstrates consistently superior performance, attributable to its policy alignment.

\begin{table}[htbp]
\caption{We evaluate the performance of our method alongside other diffusion-based baselines on the AntMaze tasks. For the baseline methods, we report the best results as presented in their original papers. The reported metrics are the average normalized scores at the end of training, along with the standard deviation across $10$ random seeds. The top two results are highlighted in bold. The prefix “D-” denotes a diffusion-based implementation.}
\label{tbl:rl_results_maze}
\centering
\small
\resizebox{\textwidth}{!}{%
\begin{tabular}{llcccccccc}
\toprule
\multicolumn{1}{c}{\bf Dataset} & \multicolumn{1}{c}{\bf Env} &  \multicolumn{1}{c}{\bf Diffusion-QL} &  \multicolumn{1}{c}{\bf QGPO} &  \multicolumn{1}{c}{\bf EDP}&  \multicolumn{1}{c}{\bf SRPO}&  \multicolumn{1}{c}{\bf DTQL}& \multicolumn{1}{c}{\bf SfBC}  & \multicolumn{1}{c}{\bf IDQL-A} & \multicolumn{1}{c}{\bf D-AlignIQL (ours)}\\
\midrule

Default       & AntMaze-umaze    & $93.4$& $\textbf{96.4}$ & $94.2$  & $\textbf{97.1}$& $94.8$ & $92.0$     & $94.0$  & $94.8$\scriptsize{±3.2} \\
Diverse       & AntMaze-umaze         & $66.2$ & $74.4$& $79.0$ & $82.1$& $78.8$& \textbf{$\textbf{85.3}$}  & $80.2$  & $\textbf{82.4}$\scriptsize{±4.4} \\
\midrule
Play          & AntMaze-medium    & $76.6$ & $83.6$ & $81.8$& $80.7$& $79.6$ & $81.3$   & $\textbf{84.5}$  & \textbf{$\textbf{87.5}$}\scriptsize{±2.5} \\
Diverse       & AntMaze-medium               & $78.6$ & $83.8$ & $82.3$& $75.0$& $82.2$& $82.0$    & $\textbf{84.8}$  & \textbf{$\textbf{85.0}$}\scriptsize{±5.0} \\
\midrule
Play          & AntMaze-large                & $46.4$ & $\textbf{66.6}$& $42.3$& $53.6$& $52.0$ & $59.3$    & $63.5$ & $\textbf{65.2}$\scriptsize{±9.6} \\
Diverse       & AntMaze-large          & $57.3$ & $64.8$ & $60.6$& $53.6$& $54.0$& $45.5$   & $\textbf{67.9}$  & $\textbf{66.4}$\scriptsize{±9.7} \\
\midrule
\multicolumn{2}{c}{\bf Average } & $69.8$ & $78.3$& $73.4$& $73.6$& $73.6$&$74.2$ & $79.1$  & $\textbf{$\textbf{80.2}$}$ \\
% \midrule
% \multicolumn{2}{c}{\bf{\# Diffusion steps}} & $5$ & $15$& $1$& $1$& $1$& $15$  & $5$  & $5$ \\
\bottomrule
\end{tabular}
}
\end{table}

\subsection{Noise Data and Vision-based Experiments}
In this section, we evaluate the benefits of policy alignment and the performance of Gaussian-based AlignIQL (hereafter referred to as AlignIQL) on noise-corrupted data and vision-based tasks. We adopt the Gaussian-based version of AlignIQL because the baselines used for comparison are based on Gaussian policies, and to demonstrate that our approach generalizes to arbitrary policy classes.

\textbf{Random Corruption.} Following \citet{yang2023towards}, we evaluate the performance of our method under various data corruption scenarios, including random perturbations to states, actions, rewards, and next-states. Corruption is introduced by adding random noise to the corrupted elements in a proportion $c$ of the dataset, with the corruption magnitude controlled by $\epsilon$. In our experiments, we set $c = \epsilon = 0.5$. Further details on the corruption process are provided in Appendix~\ref{details}.
 \begin{table*}[!htbp]
\caption{Results of Robust Experiment in Halfcheetah-medium-replay-v2. The reported metrics are the average normalized scores at the end of training, along with the
standard deviation across 5 random seeds.}
\label{table:robust}
\centering
% \small
\begin{tabular}{l|c|c|c|c|c|c}
\hline
  \multicolumn{1}{c|}{} & \multicolumn{5}{c}{\bf Halfcheetah}\\ \hline
\multicolumn{1}{c|}{\bf Method} & \multicolumn{1}{c|}{Reward} & \multicolumn{1}{c|}{Action} & \multicolumn{1}{c|}{Dynamics} & \multicolumn{1}{c|}{Observation}& \multicolumn{1}{c|}{Mix Attack} & \multicolumn{1}{c}{Average}  \\ \hline
% \midrule
\multicolumn{1}{c|}{\bf AlignIQL}  &      $40.6$\scriptsize{±0.7}       &  $40.1$\scriptsize{±1.4}   &  $37.2$\scriptsize{±9.2} &  $\textbf{30.1}$\scriptsize{±2.6}    &  $\textbf{29.1}$\scriptsize{±13.5} & $\textbf{35.4}$     \\ \hline

% \midrule
\multicolumn{1}{c|}{\bf IQL} &      $41.9$\scriptsize{±1.3}       &  $39.6$\scriptsize{±1.0} &    $\textbf{37.8}$\scriptsize{±13.4} &  $25.6$\scriptsize{±3.0} &  $24.4$\scriptsize{±12.1}  &  $33.9$    \\ \hline
\multicolumn{1}{c|}{\bf CQL} &      $\textbf{43.6}$\scriptsize{±0.8}       &  \textbf{$\textbf{44.8}$}\scriptsize{±0.8} &    $0.06$\scriptsize{±0.76} &  $28.5$\scriptsize{±16.8} &    $2.3$\scriptsize{±3.5} & $23.9$    \\ \hline
% \bottomrule
\end{tabular}
% \vspace{-0.2in}
\end{table*}

As shown in Table~\ref{table:robust}, AlignIQL achieves the highest average scores among all evaluated methods. More importantly, it exhibits superior robustness to observation attacks compared to IQL. While CQL performs well under action, observation, and reward attacks, it fails to learn under dynamics attacks. Because policy alignment depends on the value function, AlignIQL’s performance degrades under reward corruption. Nevertheless, it shows enhanced robustness to observation attacks by assigning higher weights to actions where $Q$ closely approximates $V(\vs)$. Since $V(\vs)$ is learned via a neural network, it tends to be robust to corrupted inputs (e.g., noisy observations), as similar states typically yield similar $V(\vs)$. In contrast, in reinforcement learning, $Q(\vs, \va)$ may vary substantially across similar states. This discrepancy likely explains the superior performance of AlignIQL under observation attacks.

\textbf{Vision-based Control.} To further evaluate the benefits of policy alignment and performance of AlignIQL, we report the results of AlignIQL and IQL on the Atari tasks ~\citep{agarwal2020optimisticperspectiveofflinereinforcement}. Specifically, we choose three image-based Atari games with discrete action spaces: Breakout, Qbert, and Seaquest. We use d3rlpy, a modularized offline RL library that includes several SOTA offline RL algorithms and offers an easy-to-use wrapper for the offline Atari datasets introduced by \citet{DBLP:journals/corr/abs-2108-13264}. 
To increase the task difficulty, we use only $1\%$ or $0.5\%$ of the transitions from all epochs in the original datasets. ($1\text{M}\times50\text{epoch}\times 1\% \quad\text{or}\quad 0.5\%$)

As shown in Alg~\ref{alg:iql}, the only difference between AlignIQL and IQL lies in the method of extracting policies. In Table~\ref{table:atari}, AlignIQL achieves the best performance in $5$ out of $6$ games and exhibits a smaller standard deviation compared to IQL. We also observed that in certain tasks, AlignIQL or IQL performs better on smaller datasets. This phenomenon was also observed when training CQL on Atari tasks, as reported in ~\citet{xuOfflineRLNo2023}.

 \begin{table*}[!htbp]
\caption{Performance in setting with $1\%$ or $0.5\%$ Atari dataset over $5$ random seeds. For brevity, we refer to Discrete AlignIQL as AlignIQL in this Table.}
\label{table:atari}
\centering
\small
\begin{tabular}{l|cc|cc|cc}
\hline
  & \multicolumn{2}{c|}{\bf Breakout} & \multicolumn{2}{c|}{\bf Qbert} & \multicolumn{2}{c}{\bf Seaquest}\\ 
  \hline
\multicolumn{1}{c|}{\bf Method} & \multicolumn{1}{c}{\bf $1\%$} & \multicolumn{1}{c|}{\bf $0.5\%$} 
& \multicolumn{1}{c}{\bf $1\%$} & \multicolumn{1}{c|}{\bf $0.5\%$} & \multicolumn{1}{c}{\bf $1\%$} & \multicolumn{1}{c}{\bf $0.5\%$} \\ \hline
% \midrule
\multicolumn{1}{c|}{\bf AlignIQL}  &      $\textbf{9.23}\pm 0.8$       &  $\textbf{7.13}\pm2.5$   &  $\textbf{7170}\pm877$ &  $\textbf{7512}\pm548$ &  $192.7\pm30.02$ &  $\textbf{371.3}\pm1.1$    \\ \hline

% \midrule
\multicolumn{1}{c|}{\bf IQL} &      $6.87\pm1.1$       &  $5.3\pm3.2$ &    $4160\pm1473$ &  $3773.3\pm780.2$ &  $\textbf{238.7}\pm21.6$ &  $306.7\pm25.2$    \\ \hline
% \bottomrule
\end{tabular}
% \vspace{-0.2in}
\end{table*}

\subsection{Ablation Study}
In this section, we assess the impact of different regularizers and action samples $N$ in D-AlignIQL. 

\textbf{Regularizers.} As shown in Table~\ref{table:regularizer}, we find that the performance of the linear regularizer is comparable to the results of D-AlignIQL in Table~\ref{table:regularizer}. This is because both place more weight on actions with higher $-\pbr{Q(\vs,\va)-V(\vs)}^2$. (See Appendix~\ref{details} for more details.) For $f(x)=x-1$ in D-AlignIQL-hard, we found that it is susceptible to hyperparameters, especially the learning rate of the Lagrange multiplier network, and both showed a certain decline in performance by the end of training. We attribute this performance drop to the susceptibility of the multiplier network to hyperparameters, and future improvements to the multiplier network and hyperparameters may address this issue. 

\begin{table}[ht]
    \centering
        \caption{Performance of different regularizers in D-AlignIQL and D-AlignIQL-hard. All the results are evaluated over $10$ random seeds.}
         \label{table:regularizer}

\begin{tabular}{c|c|c|c|c}
\hline
\multirow{2}{*}{Regularizers}&
\multicolumn{2}{c|}{D-AlignIQL} &
\multicolumn{2}{c}{D-AlignIQL-hard} \\
\cline{2-5}
  &umaze-p &umaze-d &umaze-p &umaze-d  \\
\hline
$f(x)=\log x$        &  $94.8$   &  $82.4$        & $84.7$ &   $74.0$  \\
\hline
$f(x)=x-1$  & $95.7$ &    $86.1$  &  $91.1$ &  $73.6$\\
\hline
\end{tabular}
   
\end{table}

\textbf{Action Samples $N$.}  As shown in Table~\ref{table:ablation}, D-AlignIQL achieves higher average performance with lower variance compared to IDQL, indicating greater robustness to variations in $N$. This robustness arises from the fact that out-of-distribution (OOD) actions generated by the policy network (e.g., via a diffusion model) may not exactly match $V(\vs)$ but can still yield high $Q(\vs, \va)$ values~\citep{fujimoto2018}. Table~\ref{table:ablation} further demonstrates that the performance of D-AlignIQL improves with increasing $N$, while IDQL does not exhibit a similar trend.

\begin{table}[!htbp]
\caption{Quantitative Results of D-AlignIQL and IDQL on AntMaze Large tasks (Play and Diverse).}
\label{table:ablation}
\centering
% \small
\begin{tabular}{l|c|c|c|c}
\hline
  & \multicolumn{1}{c|}{\bf N=16} & \multicolumn{1}{c|}{\bf N=64} &\multicolumn{1}{c}{\bf N=256}& \multicolumn{1}{c}{\bf Average}\\ \hline
\multicolumn{1}{c|}{\bf IDQL} &      $72.0$       &  $66.5$ &    $58.8$ &   65.7\scriptsize{±5.4}  \\ 
\hline
 \multicolumn{1}{c|}{\bf D-AlignIQL} & $65.8$ &  $70.2$ &  $70.7$ & 68.9\scriptsize{±2.2}\\
\hline

\end{tabular}

\end{table}
Overall, compared to IDQL, the weights computed by our method not only have better theoretical properties (applicable to any Q-loss, without requiring optimal $V$) but also perform better in practice. 

\section{Conclusion}
In our work, we define the implicit policy-finding problem in IQL and propose two practical algorithms AlignIQL-hard and AlignIQL to solve it. The optimal policy (Theorem~\ref{optimal_policy}) in AlignIQL-hard shows that it is feasible to extract policy with AWR in certain cases, which builds the bridge between the Implicit Q-learning and Weighted Regression. Our theoretical findings also extend the policy alignment of IDQL to arbitrary critic loss and value functions. Besides the theoretical findings, we also verify the effectiveness of our algorithm on D4RL datasets. Experimental results show that compared to other IQL-style algorithms, our algorithm achieves SOTA performance and is more stable, especially in sparse reward tasks. One future work is to explore better methods for training multiplier networks and explore the impact of different regularization functions of problem~\ref{align_p}. Another future work is to extend our approach to fields of safe RL and offline-to-online (O2O) learning. In safe RL, prior works~\citep{zheng2023feasibility,cao2024offline} have used IQL to learn the Q-function. Investigating how to ensure policy alignment while satisfying safety constraints is an interesting research direction.

% \textbf{Limitation.} Our work mainly focus on regularization function $f(x)=\log x$ due to its simplicity. However, different regularization functions can result in different types of policies. For example, SQL~\citep{xuOfflineRLNo2023} shows using $f(x)=x-1$ can introduce sparsity in learning the state-value function. In future work, we can explore other regularization functions to obtain better policy.

% \textbf{Boarder Impact}. Our method will promote the development of offline reinforcement learning, thus facilitating the implementation of offline reinforcement learning in practical scenarios, such as robotic control. Since our work primarily focuses on RL theory, it will not raise ethical issues.

% \begin{ack}
% Use unnumbered first level headings for the acknowledgments. All acknowledgments
% go at the end of the paper before the list of references. Moreover, you are required to declare
% funding (financial activities supporting the submitted work) and competing interests (related financial activities outside the submitted work).
% More information about this disclosure can be found at: \url{https://neurips.cc/Conferences/2025/PaperInformation/FundingDisclosure}.

% Do {\bf not} include this section in the anonymized submission, only in the final paper. You can use the \texttt{ack} environment provided in the style file to automatically hide this section in the anonymized submission.
% \end{ack}

% \section*{References}
\clearpage
\bibliographystyle{plainnat}
\bibliography{neurips_2025}

% \bibliographystyle{neurips}

%%%%%%%%%%%%%%%%%%%%%%%%%%%%%%%%%%%%%%%%%%%%%%%%%%%%%%%%%%%%
\newpage
%%%%%%%%%%%%%%%%%%%%%%%%%%%%%%%%%%%%%%%%%%%%%%%%%%%%%%%%%%%%
\appendix
\textbf{Border Impact.} Offline reinforcement learning (RL) seeks to learn a policy from a fixed dataset, analogous to supervised learning; however, challenges such as extrapolation error and function approximation make offline RL significantly more difficult than supervised learning and pre-training. Our method advances the field of offline RL by enhancing the understanding of IQL-style approaches, thereby promoting their application in real-world scenarios such as robotic control, without directly introducing substantial ethical or societal concerns.

\textbf{Limitation.} The optimality of the policy extracted by AlignIQL depends on the quality of the learned value function. Consequently, applying policy alignment with a poorly learned value function cannot yield an optimal or suboptimal policy. However, as noted in our main paper, value functions are generally learned more reliably than policies in current offline RL settings, making this issue less concerning in practice.

\section*{Index of the Appendix}
In the following, we briefly recap the contents of the Appendix.  

– Appendix A provides additional discussion about related works and extra background. 

– Appendix B reports all proofs, derivations, and some extra theoretical analysis.  

- Appendix C reports all the pseudocode of AlignIQL.

– Appendix D reports additional experiments on our method, including results of AlignIQL-hard, runtime analysis, full D4RL results, and the corresponding ablation study, along with relevant implementation details.

\section{Related Works}
\textbf{Diffusion Model in Offline RL}. Due to our method using the diffusion model for modeling behavior policy, we review works that incorporate the Diffusion model in offline RL. There exist several works that introduce the diffusion model to RL. Diffuser~\citep{janner2022} uses the diffusion model to directly generate trajectory guided with gradient guidance or reward. DiffusionQL~\citep{wang2022b} uses the diffusion model as an actor and optimizes it through the TD3+BC-style objective with a coefficient $\eta$ to balance the two terms. AdaptDiffuser~\cite {liang2023} uses a diffusion model to generate extra trajectories and a discriminator to select desired data to add to the training set to enhance the adaptability of the diffusion model. DD~\citep{ajay2022}  uses a conditional diffusion model to generate a trajectory and compose skills. Unlike Diffuser, DD diffuses only states and trains inverse dynamics to predict actions. QGPO~\cite{lu2023} uses the energy function to guide the sampling process and proves that the proposed CEP training method can get an unbiased estimation of the gradient of the energy function under unlimited model capacity and data samples. SfBC~\citep{chen2022} first trains a diffusion-based policy and then selects actions based on the $Q$ value, similar to AWR.IDQL~\citep{hansen-estruch2023} reinterpret IQL as an Actor-Critic method and extract the policy through sampling from a diffusion-parameterized behavior policy with weights computed from the IQL-style critic. EDP~\citep{kang_efficient_nodate} focuses on boosting sampling speed through approximated actions. SRPO~\citep{chen_score_2023} uses a Gaussian policy in which the gradient is regularized by a pretrained diffusion model to recover the IQL-style policy. DTQL~\citep{chen2024diffusion} distills DiffusionQL into a one-step policy using a diffusion trust region loss. Our method is distinct from these methods because we aim to
align the implied policy with the value function. 

\subsection{Diffusion model}
%%%%%%%%%%%%%%%%%%%%%%%%%%%%%%%%%%%%%%%%%%%%%%%%%%%%%%%%%%%%%%%%%%%%%%%%%%%%%%%%%%%
%%%%%%%%%%%%%%%%%%%%%%%%%%%%%%%%%%%%%%%%%%%%%%%%%%%%%%%%%%%%%%%%%%%%%%%%%%%%%%%%%%%
%%%%%%%%%%%%%%%%%%%%%%%%%%%%%%%%%%%%%%%%%%%%%%%%%%%%%%%%%%%%%%%%%%%%%%%%%%%%%%%%%%%
%%%%%%%%%%%%%%%%%%%%%%%%%%%%%%%%%%%%%%%%%%%%%%%%%%%%%%%%%%%%%%%%%%%%%%%%%%%%%%%%%%%
%%%%%%%%%%%%%%%%%%%%%%%%%%%%%%%%%%%%%%%%%%%%%%%%%%%%%%%%%%%%%%%%%%%%%%%%%%%%%%%%%%%
\textbf{Diffusion Probabilistic Model (DPM).} 
Diffusion models~\citep{sohl-dickstein2015,ho2020,song2019} are composed of two processes: the forward diffusion process and the reverse process. In the forward diffusion process, we gradually add Gaussian noise to the data $\vx_0 \sim q(\vx_0)$ in $T$ steps. The step sizes are controlled by a variance schedule $\beta_i$:
\begin{equation}
\begin{aligned}
                &q(\vx^{1:T}  \,|\,  \vx^0) :=  \textstyle \prod_{i=1}^T q(\vx^i  \,|\,  \vx^{i-1}), \\
                &q(\vx^i  \,|\,  \vx^{i-1}) := \gN (\vx^i; \sqrt{1 - \beta_i} \vx^{i-1}, \beta_i %
        \mI).
\end{aligned}
\end{equation}

In the reverse process, we can recreate the true sample $\vx_0$ through $p(\vx^{i-1}\vert\vx^{i})$:
\begin{equation}
    \begin{aligned}
            p(\vx) & = \int p(\vx^{0:T})d\vx^{1:T}\\
            &=\int\gN (\vx^T; \mathbf{0},\mI)\prod_{i=1}^Tp(\vx^{i-1}\vert\vx^{i})d\vx^{1:T}.
    \end{aligned}
    \label{reverse}
\end{equation}
The training objective is to maximize the ELBO of $\E_{\vq_{x_0}}\sbr{\log p(\vx_0)}$. Following DDPM~\citep{ho2020}, we  use the simplified surrogate loss 
\begin{equation}
\label{diffusion_loss}
    \gL_d(\phi) = \E_{i \sim \sbr{1,T}, \rvepsilon \sim \gN(\mathbf{0}, \mI), \vx_0 \sim q} \sbr{|| \rvepsilon - \rvepsilon_\phi(\vx_i, i) ||^2}
\end{equation}

to approximate the ELBO. After training, sampling from the diffusion model is equivalent to running the reverse process. 

\textbf{Conditional DPM.} 
There are two kinds of conditioning methods: classifier-guided~\citep{dhariwal2021} and classifier-free~\citep{ho2022}. The former requires training a classifier on noisy data $\vx_i$ and using gradients $\nabla_\vx\log f_\Phi(\vy\vert\vx_i)$ to guide the diffusion sample toward the conditioning information $\vy$. The latter does not train an independent $f_\Phi$ but combines a conditional noise model $\epsilon_\phi(\vx_i, i,\vs)$ and an unconditional model $\epsilon_\phi(\vx_i, i)$ for the noise. The perturbed noise $w\epsilon_\phi(\vx_i, i)+(w+1)\epsilon_\phi(\vx_i,i,\vs)$ is used to later generate samples. However \citep{pearce2023imitating} shows this combination will degrade the policy performance in offline RL. Following~\citep{pearce2023imitating,wang2022b} we solely employ a conditional noise model $\epsilon_\phi(\vx_i,i,\vs)$ to construct our noise model ($w=0$).

\subsection{Implicit Diffusion Q-learning (IDQL)}
\label{sec:idql}
\textbf{Implicit Diffusion Q-learning (IDQL)}. To find the implicit policy in the learned value function, IDQL~\citep{hansen-estruch2023} generalizes the value loss in \Eqref{eqn:fit_v_expectiles} with an arbitrary convex loss $U$ on the difference $Q-V$. 
\begin{equation}
\label{eqn:general_IQL}
    V^*(\vs) = \argmin_{V(\vs)} \mathbb{E}_{\va \sim \mu(\va|\vs)}[U(Q(\vs,\va) - V(\vs))] = \argmin_{V(\vs)} \mathcal{L}_{V}^U(V(\vs)).
\end{equation}
Under some assumptions about $U$, IDQL derives the implicit policy in optimal $V$ defined in \Eqref{eqn:general_IQL} 
\begin{equation}
\label{eqn:importance_weight}
    w(\vs,\va) = \frac{|U'(Q(\vs,\va) - V^*(\vs))|}{|Q(\vs,\va) - V^*(\vs)|},
\end{equation}
which yields an expression for the implicit actor as $\pi_{\text{imp}}(\va|\vs) \propto \mu(\va|\vs) w(\vs,\va)$. 
For expectile loss $f(u) = L_2^\tau(u)$ (from \Eqref{eqn:fit_v_expectiles}), the weight of IDQL is 
\begin{equation}
\label{eq:Expectile_Weights}
    w_2^{\tau}(s, a) = |\tau - \mathbbm{1}(Q(s,a) < V_\tau^2(s))|.
\end{equation}

%%%%%%%%%%%%%%%%%%%%%%%%%%%%%%%%%%%%%%%%%%%%%%%%%%%%%%%%%%%%

\section{Proofs and Derivations}

\label{all_prof}
\subsection{Proof of Theorem~\ref{optimal_policy}}
\label{theorem1}
\begin{proof}
The Lagrange function of \Eqref{align_p} is written as follows
\begin{equation}
\label{lagrange}
    \begin{aligned}
        L(\pi,\alpha(\vs),\beta(\vs) ,\lambda)&=\E_{\vs\sim d^\pi(\vs),\va\sim\pi(\va|\vs)}\sbr{f(\frac{\pi(\va|\vs)}{\mu(\va|\vs)})}-\E_{\vs\sim d^\pi(\vs),\va\sim\pi(\va|\vs)}\sbr{\lambda(\va|\vs)\pi(\va|\vs)}\\
        &+\E_{\vs\sim d^\pi(\vs)}\sbr{\alpha(\vs)\pbr{\int_\va\pi(\va|\vs)d\va -1}}+\\
        &\E_{\vs\sim d^\pi(\vs)}\sbr{\beta(\vs)\pbr{\E_{\va\sim\pi(\va|\vs)}\sbr{Q(\vs,\va)}-V(\vs)}},
    \end{aligned}
\end{equation}
where $d^\pi(\vs)$ represents the state distribution induced by policy $\pi$, $\alpha(\vs)$, $\beta(\vs) $, and $\lambda$ are Lagrangian multipliers for the equality and inequality constraints respectively.

Let $h_f(x)=xf(x)$. Then for all states and actions, the KKT conditions can be written as follows
\begin{align}
    &\pi(\va|\vs)\geq0\\
    \label{c1}
    &\int_\va\pi(\va|\vs)d\va=1\\
    \label{c2}
    &\E_{\va\sim\pi(\va|\vs)}\sbr{Q(\vs,\va)-V(\vs)}=0\\
    &\lambda(\va|\vs)\geq0\\
    \label{comple}
    &\lambda(\va|\vs)\pi(\va|\vs)=0\\
    \label{diff}
    &h_f'(\frac{\pi(\va|\vs)}{\mu(\va|\vs)})+\alpha(\vs)+\beta(\vs) Q(\vs,\va)-\lambda(\va|\vs)=0
\end{align}
We eliminate $d^\pi(\vs)$ due to irreducible Markov chain assumption. Note that in our derivation, we assume that $V(\vs)$ and $Q(\vs,\va)$ are known.

Since $h_f'$ is a strictly increasing function, its inverse function exists and is also a strictly increasing function. Let $g_f=(h_f')^{-1}(x)$ be its inverse function. 
From \Eqref{diff}, we can get 
\begin{equation}
    \pi(\va\vert\vs) = \mu(\va\vert\vs)g_f\pbr{\lambda(\va|\vs)-\alpha(\vs)-\beta(\vs) Q(\vs,\va)}
\end{equation}

Given a state $\vs$,  we can get $\lambda(\va|\vs)=h_f'(\frac{\pi}{\mu})+\alpha(\vs)+\beta(\vs) Q(\vs,\va)$ from \Eqref{diff}, then 
\begin{enumerate}
    \item[(a)] If $\lambda(\va|\vs)=h_f'(\frac{\pi}{\mu})+\alpha(\vs)+\beta(\vs) Q(\vs,\va)>0$, then $\pi(\va|\vs)$ is zero  due to  complementary slackness. Note that $\pi(\va|\vs)=0$, thus $h_f'(0)+\alpha(\vs)+\beta(\vs) Q(\vs,\va)>0$ and we can get $g_f(-\alpha(\vs)-\beta(\vs) Q(\vs,\va))<g_f(h_f'(0))=0$. 
    
    \item[(b)] If $\lambda(\va|\vs)=0$, then $h_f'(\frac{\pi}{\mu})+\alpha(\vs)+\beta(\vs) Q(\vs,\va)$ is zero and $\pi(\va|\vs)=\mu(\va\vert\vs)g_f\pbr{-\alpha(\vs)-\beta(\vs) Q(\vs,\va)}\geq0$. Note that $\pi(\va|\vs)\geq0$, thus $h_f'(0)+\alpha(\vs)+\beta(\vs) Q(\vs,\va)\leq0$ and we can get $g_f(-\alpha(\vs)-\beta(\vs) Q(\vs,\va))\geq g_f(h_f'(0))=0$.
\end{enumerate}

Through analysis (a) and (b), we can resolve optimal policy $\pi^*(\va|\vs)$ as 
\begin{equation}
    % \label{pi_star}
    \pi^\star(\va\vert\vs) = \mu(\va\vert\vs)\max\cbr{g_f\pbr{-\alpha(\vs)-\beta(\vs) Q(\vs,\va)},0}.
\end{equation}

Substituting back in \Eqref{c1} and \Eqref{c2} with \Eqref{pi_star}, we can get
\begin{equation}
    % \label{loss_1}
    \E_{\va\sim\mu}\sbr{\max\cbr{g_f(-\alpha^*(\vs)-\beta^*(\vs) Q(\vs,\va)),0}}=1,
\end{equation}
\begin{equation}
    % \label{loss_2}
    \E_{\va\sim\mu(\va|\vs)}\sbr{Q(\vs,\va)\max\cbr{g_f(-\alpha^*(\vs)-\beta^*(\vs) Q(\vs,\va)),0}-V(\vs)}=0,
\end{equation}
\end{proof}
%%%%%%%%%%%%%%%%%%%%%%%%%%%%%%%%%%%%%%%%%%%%%%%%%%%%%%%%%%%%%%%%%%%%%%%%%%%%%%%%%%%
%%%%%%%%%%%%%%%%%%%%%%%%%%%%%%%%%%%%%%%%%%%%%%%%%%%%%%%%%%%%%%%%%%%%%%%%%%%%%%%%%%%
%%%%%%%%%%%%%%%%%%%%%%%%%%%%%%%%%%%%%%%%%%%%%%%%%%%%%%%%%%%%%%%%%%%%%%%%%%%%%%%%%%%
%%%%%%%%%%%%%%%%%%%%%%%%%%%%%%%%%%%%%%%%%%%%%%%%%%%%%%%%%%%%%%%%%%%%%%%%%%%%%%%%%%%
%%%%%%%%%%%%%%%%%%%%%%%%%%%%%%%%%%%%%%%%%%%%%%%%%%%%%%%%%%%%%%%%%%%%%%%%%%%%%%%%%%%
\subsection{Proof of Theorem~\ref{soft-optimal}}
\label{theorem2}
\begin{proof}
The Lagrange function of \Eqref{align_soft} is written as follows
\begin{equation}
\label{lagrange_s}
    \begin{aligned}
        L(\pi,V,\alpha(\vs),\beta(\vs) ,\lambda)&=\E_{\vs\sim d^\pi(\vs),\va\sim\pi(\va|\vs)}\sbr{f(\frac{\pi(\va|\vs)}{\mu(\va|\vs)})+\eta\pbr{Q(\vs,\va)-V(\vs)}^2}\\
        &-\E_{\vs\sim d^\pi(\vs),\va\sim\pi(\va|\vs)}\sbr{\lambda(\va|\vs)\pi(\va|\vs)}\\
        &+\E_{\vs\sim d^\pi(\vs)}\sbr{\alpha(\vs)\pbr{\int_\va\pi(\va|\vs)d\va -1}}.
    \end{aligned}
\end{equation}

Let $h_f(x)=xf(x)$. Then for all states and actions, the KKT conditions can be written as follows
\begin{align}
    &\pi(\va|\vs)\geq0\\
    \label{c_s1}
    &\int_\va\pi(\va|\vs)d\va=1\\
    \label{c_S2}
    &\E_{\va\sim\pi(\va|\vs)}\sbr{Q(\vs,\va)-V(\vs)}=0\\
    &\lambda(\va|\vs)\geq0\\
    \label{comple_s}
    &\lambda(\va|\vs)\pi(\va|\vs)=0\\
    \label{diff_s}
    &h_f'(\frac{\pi(\va|\vs)}{\mu(\va|\vs)})+\alpha(\vs)+\eta\pbr{Q(\vs,\va)-V(\vs)}^2-\lambda(\va|\vs)=0
\end{align}

Since $h_f'$ is a strictly increasing function, its inverse function exists and is also a strictly increasing function. Let $g_f=(h_f')^{-1}(x)$ be its inverse function. 
From \Eqref{diff_s}, we can get 
\begin{equation}
    \pi(\va\vert\vs) = \mu(\va\vert\vs)g_f\pbr{\lambda(\va|\vs)-\alpha(\vs)-\eta\pbr{Q(\vs,\va)-V(\vs)}^2}
\end{equation}

Given a state $\vs$,  we can get $\lambda(\va|\vs)=h_f'(\frac{\pi}{\mu})+\alpha(\vs)+\eta\pbr{Q(\vs,\va)-V(\vs)}^2$ from \Eqref{diff_s}, then 
\begin{enumerate}
    \item[(a)] If $\lambda(\va|\vs)=h_f'(\frac{\pi}{\mu})+\alpha(\vs)+\eta\pbr{Q(\vs,\va)-V(\vs)}^2>0$, then $\pi(\va|\vs)$ is zero  due to  complementary slackness. Note that $\pi(\va|\vs)=0$, thus $h_f'(0)+\alpha(\vs)+\eta\pbr{Q(\vs,\va)-V(\vs)}^2>0$ and we can get $g_f(-\alpha(\vs)-\eta\pbr{Q(\vs,\va)-V(\vs)}^2)<g_f(h_f'(0))=0$. 
    
    \item[(b)] If $\lambda(\va|\vs)=0$, then $h_f'(\frac{\pi}{\mu})+\alpha(\vs)+\eta\pbr{Q(\vs,\va)-V(\vs)}^2$ is zero and $\pi(\va|\vs)=\mu(\va\vert\vs)g_f\pbr{-\alpha(\vs)-\eta\pbr{Q(\vs,\va)-V(\vs)}^2}\geq0$. Note that $\pi(\va|\vs)\geq0$, thus $h_f'(0)+\alpha(\vs)+\eta\pbr{Q(\vs,\va)-V(\vs)}^2\leq0$ and we can get $g_f(-\alpha(\vs)-\eta\pbr{Q(\vs,\va)-V(\vs)}^2)\geq g_f(h_f'(0))=0$.
\end{enumerate}

Through analysis (a) and (b), we can resolve optimal policy $\pi^*(\va|\vs)$ as 
\begin{equation}
    \label{pi_s_star}
    \pi^\star(\va\vert\vs) = \mu(\va\vert\vs)\max\cbr{g_f\pbr{-\alpha(\vs)-\eta\pbr{Q(\vs,\va)-V(\vs)}^2},0}.
\end{equation}

let $f(x)=\log x$, then $g_f(x)=\exp{(x-1))}>0$. Substituting back in \Eqref{pi_s_star} with $g_f(x)=\exp{(x-1))}$, we can get \Eqref{pi_soft}.
\end{proof}

%%%%%%%%%%%%%%%%%%%%%%%%%%%%%%%%%%%%%%%%%%%%%%%%%%%%%%%%%%%%%%%%%%%%%%%%%%%%%%%%%%%
%%%%%%%%%%%%%%%%%%%%%%%%%%%%%%%%%%%%%%%%%%%%%%%%%%%%%%%%%%%%%%%%%%%%%%%%%%%%%%%%%%%
%%%%%%%%%%%%%%%%%%%%%%%%%%%%%%%%%%%%%%%%%%%%%%%%%%%%%%%%%%%%%%%%%%%%%%%%%%%%%%%%%%%
%%%%%%%%%%%%%%%%%%%%%%%%%%%%%%%%%%%%%%%%%%%%%%%%%%%%%%%%%%%%%%%%%%%%%%%%%%%%%%%%%%%
%%%%%%%%%%%%%%%%%%%%%%%%%%%%%%%%%%%%%%%%%%%%%%%%%%%%%%%%%%%%%%%%%%%%%%%%%%%%%%%%%%%
\subsection{Proof of Proposition~\ref{connect}}
\label{proposition}
\begin{proof}
The proof of Proposition~\ref{connect} is based on finding a minimum for the Problem~\ref{align_soft} in a region, and then let the value of Problem~\ref{align_soft} at $\pi^*, V^*$  less than the minimum of Problem~\ref{align_soft} in this region to determine the value of $\eta$.
Let $\gU=\cbr{\pi(\va\vert\vs)\vert \pi(\va\vert\vs)\geq0,\int_\va\pi(\va\vert\vs)d\va=1}$.
Since $\inf_{x,y}g(x,y)=\inf_x\inf_yg(x,y)$ and the constraints about  $\pi$ and $V$ in Problem~\ref{align_soft} are independent, we can reformulate Problem~\ref{align_soft} as 
\begin{equation}
\label{step1}
\begin{aligned}
        &\min_{\substack{\pi,V \\ s.t. \pi\in\gU}}  \E_{\vs\sim d^\pi(\vs),\va\sim\policy}\sbr{f\pbr{\frac{\policy}{\mu(\va\vert\vs)}}+\eta\pbr{Q(\vs,\va)-V(\vs)}^2} \\
        &= \min_{\substack{\pi \\ s.t. \pi\in\gU}}\min_{V}  \E_{\vs\sim d^\pi(\vs),\va\sim\policy}\sbr{f\pbr{\frac{\policy}{\mu(\va\vert\vs)}}+\eta\pbr{Q(\vs,\va)-V(\vs)}^2}.
\end{aligned}
\end{equation}
For $V$, this is an unconstrained problem. Setting the gradient with respect to $V$ to $0$ ($\eta>0$), we obtain that
\begin{equation}
    V(\vs)=\E_{\va\sim\pi(\va|\vs)}\sbr{Q(\vs,\va)}.
\end{equation}
Substituting back in \Eqref{step1}, we can get 
\begin{equation}
\label{step2}
\begin{aligned}
        & \min_{\substack{\pi \\ s.t. \pi\in\gU}}\min_{V}  \E_{\vs\sim d^\pi(\vs),\va\sim\policy}\sbr{f\pbr{\frac{\policy}{\mu(\va\vert\vs)}}+\eta\pbr{Q(\vs,\va)-V(\vs)}^2}\\
        & = \min_{\substack{\pi \\ s.t. \pi\in\gV}}  \E_{\vs\sim d^\pi(\vs),\va\sim\policy}\sbr{f\pbr{\frac{\policy}{\mu(\va\vert\vs)}}+\eta\pbr{Q(\vs,\va)-V(\vs)}^2},\\
\end{aligned}
\end{equation}
where $\gV=\cbr{\pi(\va\vert\vs)\vert\pi(\va\vert\vs)\in\gU,V(\vs)=\E_{\va\sim\pi(\va|\vs)}\sbr{Q(\vs,\va)}}$. Note that $\gV$ is the feasible set of Problem~\ref{align_p} and the left-hand side of \Eqref{step2} is exactly Problem~\ref{align_p}.

Let $\gT=\cbr{\pi|\pi\in\gV,\pi\in\mathring{U}(\pi^*,\sigma)}$,  where $\mathring{U}(\pi^*,\sigma)=\cbr{\pi|0<|\pi-\pi^*|<\sigma,\sigma>0}$. Note that $\gT\notin\emptyset$, since $\gV$ is a convex set and $\pi^*\in\gV$. 

Assume \Eqref{step2} can achieve the minimum in $\gT$; if it cannot, it indicates a minimum at $\pi^*$ and $V^*$, and Proposition~\ref{connect} holds. We only need to adjust the value of $\eta$ to ensure that the value of \Eqref{step2} at $\pi^*, V^*$ is less than the minimum, thereby proving Proposition~\ref{connect}.

\begin{equation}
\label{step3}
\begin{aligned}
        &  k^*=\min_{\substack{\pi \\ s.t. \pi\in\gT}}  \E_{\substack{\vs\sim d^\pi(\vs)\\ \va\sim\policy}}\sbr{f\pbr{\frac{\policy}{\mu(\va\vert\vs)}}+\eta\pbr{Q(\vs,\va)-V(\vs)}^2}\\
\end{aligned}
\end{equation}

Therefore, if the value of \Eqref{step2} at $\pi^*,V^*$ is less than $k^*$, then  $\pi^*,V^*$ is a  local minimizer of Problem~\ref{align_soft}. Let $h^*=\E_{\substack{\vs\sim d^\pi(\vs)\\\va\sim\pi^*(\va|\vs)}}\sbr{\pbr{Q(\vs,\va)-V^*(\vs)}^2}$, we can get 
\begin{equation}
\label{eta}
    p^*+\eta h^* = k^*.
\end{equation}
where $p^*$ is the global solution of problem~\ref{align_p}. 
Here, for simplicity, we treat $\eta$ as a hyperparameter rather than solving for its exact value. So if $\eta$ satisfies \Eqref{eta}, we can get $\pi^*,V^*$ is a local minimizer of Problem~\ref{align_p}.

% Since $\inf \cbr{f(x)+g(x)} \geq \inf f(x)+\inf g(x)$, we can get 

% The LHS is exactly Problem~\ref{align_s}, which has the global optimal value at $\pi^*, V^*$, so that 
% \begin{equation}
% \label{step4}
% \begin{aligned}
%         & \min_{\substack{\pi \\ s.t. \pi\in\gV}}  \E_{\substack{\vs\sim d^\pi(\vs)\\ \va\sim\policy}}\sbr{f\pbr{\frac{\policy}{\mu(\va\vert\vs)}}}+\min_{\substack{\pi \\ s.t. \pi\in\gV}}\E_{\substack{\vs\sim d^\pi(\vs)\\ \va\sim\policy}}\sbr{\eta\pbr{Q(\vs,\va)-V(\vs)}^2}\\
%         &\geq p^* +\min_{\substack{\pi \\ s.t. \pi\in\gV}}\E_{\substack{\vs\sim d^\pi(\vs)\\ \va\sim\policy}}\sbr{\eta\pbr{Q(\vs,\va)-V(\vs)}^2},
% \end{aligned}
% \end{equation}
%  Suppose that we have 

% Let $k^*\eta=\inf\cbr{\eta\E_{\substack{\vs\sim d^\pi(\vs)\\ \va\sim\policy}}\sbr{\pbr{Q(\vs,\va)-V(\vs)}^2}\vert\pi\in\gT}$, we can get $p^*+\eta k^*$ is the minimum  for $\pi\in\gT,\forall V$.

\end{proof}

%%%%%%%%%%%%%%%%%%%%%%%%%%%%%%%%%%%%%%%%%%%%%%%%%%%%%%%%%%%%%%%%%%%%%%%%%%%%%%%%%%%
%%%%%%%%%%%%%%%%%%%%%%%%%%%%%%%%%%%%%%%%%%%%%%%%%%%%%%%%%%%%%%%%%%%%%%%%%%%%%%%%%%%
%%%%%%%%%%%%%%%%%%%%%%%%%%%%%%%%%%%%%%%%%%%%%%%%%%%%%%%%%%%%%%%%%%%%%%%%%%%%%%%%%%%
%%%%%%%%%%%%%%%%%%%%%%%%%%%%%%%%%%%%%%%%%%%%%%%%%%%%%%%%%%%%%%%%%%%%%%%%%%%%%%%%%%%
%%%%%%%%%%%%%%%%%%%%%%%%%%%%%%%%%%%%%%%%%%%%%%%%%%%%%%%%%%%%%%%%%%%%%%%%%%%%%%%%%%%

\subsection{Extra Theoretical Analysis}
In this section, we provide additional theoretical analysis on the time complexity of AlignIQL and AlignIQL-hard, as well as the suboptimality gap between the IPF and IPF-Soft formulations.

\textbf{Suboptimality Gap.} We compute the KL divergence between the solutions of IPF (i.e., AlignIQL-hard) and IPF-Soft (i.e., AlignIQL) to investigate the suboptimality introduced by our relaxation. Assume that $Q(\vs, \va)$ and $V(\vs)$ are given, and let $\pi^*(\va|\vs)$ and $\hat{\pi}^*(\va|\vs)$ denote the optimal solutions of IPF and IPF-Soft as derived in Theorems 5.5 and 5.1, respectively.

For the regularization function $f(x) = \log x$, we obtain

\begin{equation}
    \begin{aligned}
        &D_{KL}\pbr{\hat{\pi}^{*}(\va|\vs)|\pi^{*}(\va|\vs)}\\
        &=\int \hat{\pi}^{*}(\va|\vs)\log\frac{k(\vs)\exp{\pbr{-\eta(Q-V)^2}}}{\exp{\pbr{-\beta Q}}}d\va, \\
    \end{aligned}
    \label{suboptimal_gap1}
\end{equation}
where $k(\vs)$ is the normalized ratio to keep $\pi^*(\va|\vs)$ and  $\hat{\pi}^*(\va|\vs)$ are distributions. 
\begin{equation}
    \begin{aligned}
        &=\int \hat{\pi}^{*}(\va|\vs)\log\frac{k(\vs)\exp{\pbr{-\eta(Q-V)^2}}}{\exp{\pbr{-\beta Q}}}d\va \\
        &= \int \hat{\pi}^{*}(\va|\vs)\pbr{\log{k}-\eta(Q-V)^2+\beta Q}d\va \\
        &= log{k} - \underbrace{\eta\int \hat{\pi}^{*}(\va|\vs)(Q-V)^2 d \va}_K +\underbrace{\beta\int \hat{\pi}^{*}(\va|\vs)Q d \va}_M \\
    \end{aligned}
    \label{suboptimal_gap2}
\end{equation}
Suppose that 
\begin{equation}
     V(\vs)=\E_{\va\sim\hat{\pi}^{*}(\va|\vs)}\sbr{Q(\vs,\va)},
\end{equation}
which is one of the optimal conditions to IPF-soft according to A.2, then if we treat $Q(\vs, \va)$ as a function of the random variable $\va$ drawn from the policy $\hat{\pi}^*(\va|\vs)$, the gap between the solutions of IPF and IPF-Soft is influenced by the variance of $Q(\vs, \va)$ ($K$), as well as by $M$, $\eta$, and the Lagrange multiplier $\beta$. Since the KL divergence is non-negative, reducing the suboptimality gap requires minimizing the negative terms $-K$ and $M$ (when $\beta M < 0$). Given $\eta > 0$, a higher variance of $Q$ leads to a smaller gap. This implies that a larger $\sup |Q(\vs, \va)|$ helps reduce the suboptimality, while $\beta$ serves as a regularizer that penalizes over- or underestimation of $Q$. This highlights the importance of accurately estimating $Q$ to control the suboptimality gap. Furthermore, according to \Eqref{suboptimal_gap2}, the gap can be made small by tuning $\eta$, and importantly, $\eta$ need not go to infinity, which is consistent with our Proposition 5.6.

\textbf{Time Complexity.} In this part, we analyze the time complexity of AlignIQL and AlignIQL-hard. Since IPF is a convex problem, if the KKT conditions (Equations $30$-$35$) can be solved exactly, a closed-form solution for AlignIQL-hard is attainable. However, this solution requires access to the Lagrange multipliers, which we approximate using a neural network in our implementation (Lemma 5.2). As shown in Table~\ref{tab:hy1}, we use Adam to optimize this multiplier network. Thanks to the decoupled training in IQL, where the policy and critic networks are trained independently, the extra time complexity of D-AlignIQL-hard arises solely from optimizing the multiplier network. Assuming we aim to approximate the optimal IPF policy $\pi^*$ with $\pi_\phi$, for the regularizer $f(x) = \log x$, the KL divergence between $\pi_\phi$ and the optimal policy $\pi^*$ can be bounded by

\begin{equation}
    \begin{aligned}
        &D_{KL}\pbr{\pi^{*}(\va|\vs)|\pi_\phi(\va|\vs)}\\
        &=\int \pi^{*}(\va|\vs)\log\frac{\exp{\pbr{-\alpha(\vs)-\beta(\vs)Q -1}}}{\exp{\pbr{-\alpha_\phi(\vs)-\beta_\phi(\vs) Q-1}}}d\va, \\
        &= \int \pi^{*}(\va|\vs)\pbr{\alpha_\phi(\vs)-\alpha(\vs)+Q(\beta_\phi(\vs) -\beta(\vs))}d\va \\
        &\leq \int \pi^{*}(\va|\vs)|\alpha_\phi(\vs)-\alpha(\vs)+Q(\beta_\phi(\vs) -\beta(\vs))|d\va \\
        &\leq \int \pi^{*}(\va|\vs)\pbr{|\alpha_\phi(\vs)-\alpha(\vs)|+|Q(\beta_\phi(\vs) -\beta(\vs))|}d\va \\
    \end{aligned}
    \label{time_complex1}
\end{equation}

According to Lemma 5.2, the true Lagrange multipliers correspond to the stationary points of Eq.\eqref{loss_multiplier}, $i.e$, points where the gradient vanishes. This implies that the time complexity of obtaining an $\epsilon$-suboptimal solution depends on the convergence rate of the optimizer to a (local) optimum. In our implementation, we use the Adam optimizer with default parameters. As shown in \cite{defossez2020simple}, under certain assumptions, Adam achieves a convergence rate of $\mathcal{O}(d\ln{N}/\sqrt{N})$, where $d$ is the dimensionality and $N$ is the total number of iterations. Therefore, based on \Eqref{time_complex1}, the additional time complexity introduced by the multiplier network in AlignIQL-hard is also approximately $\mathcal{O}(d\ln{N}/\sqrt{N})$.

\begin{table}[htbp]
\caption{Runtime of different diffusion-based offline RL methods. (Average)}
\centering
\small
\scalebox{0.9}{
\begin{tabular}{lcccccccc}
\toprule
\multicolumn{1}{c}{\bf D4RL Tasks}  & \multicolumn{1}{c}{\bf D-AlignIQL (ours) (T=5) }
& \multicolumn{1}{c}{\bf D-AlignIQL-hard (ours)}
& \multicolumn{1}{c}{\bf SfBC (T=5)} & \multicolumn{1}{c}{\bf IDQL (T=5)}   \\
\midrule

\multicolumn{1}{c}{\bf Umaze Runtime ($1$ epoch)} &      $4.2$s       &  $4.6$s &    $4.3$s    &  $4.5$s \\

\bottomrule
\end{tabular}
}
\label{table:time2}
\end{table}

For AlignIQL, we obtain a closed-form solution of IPF-Soft, i.e., AlignIQL itself. This closed-form solution does not require explicit computation of multipliers, and thus incurs no additional cost for solving IPF-Soft. This is the core motivation for relaxing the IPF constraints.

In Appendix~\ref{runtime}, we report empirical runtime evaluations of D-AlignIQL. Building on this, we additionally test the runtime of D-AlignIQL-hard in the AntMaze-umaze environment to assess its practical time complexity. Table~\ref{table:time2} shows that although AlignIQL-hard introduces an additional multiplier network, it still matches the runtime of other methods.

\section{Pseudocode}
\label{pesudocode}
\begin{figure}[t]
\centering
\begin{minipage}[t]{0.48\textwidth}
\centering
\begin{algorithm}[H]
\small
    \caption{AlignIQL Training}
    \begin{algorithmic}[1]
    \label{alg:training}
    \STATE Initialize behavior policy network $\mu_{\phi}$, critic networks $Q_{\theta}$,$V_\psi$, and target networks  $Q_{\hat{\theta}}$, multiplier networks $\alpha_\omega(\vs),\beta_\chi(\vs)$\! %$\mu_{\hat{\phi}}$,
    \FOR{$t=1$ to $T$}
    \STATE Sample from $\gB\!=\!\cbr{(\vs_t, \va_t, r_t, \vs_{t+1})}\!\sim\!\gD$.
    \STATE \textit{\bfseries \# Critic updating}
    \STATE $\psi \leftarrow \psi - \lambda \nabla_{\psi} \mathcal{L}_V(\psi)$ (\Eqref{eqn:fit_v_expectiles})
    \STATE $\theta \leftarrow \theta - \lambda \nabla_{\theta} \mathcal{L}_Q(\theta)$ (Equation~\ref{eqn:fit_q})
    \IF{AlignIQL-hard:}
    \STATE \textit{\bfseries \# Multiplier network updating}
    \STATE $\omega \leftarrow \omega+\lambda \nabla_{\omega} \mathcal{L}_M(\omega)$ 
    \STATE $\chi \leftarrow \chi+ \lambda \nabla_{\chi} \mathcal{L}_M(\chi)$
    \ENDIF
    \STATE $\phi \leftarrow \phi - \lambda \nabla_{\phi} \mathcal{L}_{\mu}(\phi)$(\Eqref{diffusion_loss})
    \STATE \textit{\bfseries \# Target Networks updating}
    \STATE $\hat \theta \leftarrow (1 - \eta) \hat \theta + \eta \theta$   
    \ENDFOR
    \end{algorithmic}
\end{algorithm}
\end{minipage}
\centering
\begin{minipage}[t]{0.48\textwidth}
\centering
\begin{algorithm}[H]
\small
    \caption{AlignIQL Policy Extraction}
    
    \begin{algorithmic}[1]
    \label{alg:extraction}
    \STATE \textbf{Pretraining:} $Q_{\hat{\theta}}$,$V_\psi$,$\mu_{\phi}$,multiplier networks $\alpha_\omega(\vs),\beta_\chi(\vs)$
    \STATE Samples per state $N$, $\eta$
    \WHILE{\text{not done}}
    \STATE Get current state $\vs$
    \STATE Sample $a_i \sim \mu_{\phi}(\va|\vs)$, $i = 1, \ldots, N$
    \IF{AlignIQL-hard:}
    \STATE Compute weight $w(\vs,\va)$ through \Eqref{pi_star}
    \ELSE
    \STATE Compute weight $w(\vs,\va)$ through \Eqref{pi_soft}
    \ENDIF
    \STATE Normalize:  $p_i = \frac{w(s, a_i)}{\sum_j w(s, a_j)}$
    \STATE Select $a_{\text{taken}}$ with the highest probability according to  $p_i$
    \ENDWHILE
    \end{algorithmic}
\end{algorithm}
\end{minipage}
\end{figure}

\begin{algorithm}[H]
            \caption{IQL using AlignIQL or AWR}
            \begin{algorithmic}
            \label{alg:iql}
                \STATE Initialize parameters $\psi$, $\theta$, $\hat{\theta}$, $\phi$.
                \STATE TD learning (IQL):
                \FOR{each gradient step}
                \STATE $\psi \leftarrow \psi - \lambda_V \nabla_\psi L_V(\psi)$ 
                \STATE $\theta \leftarrow \theta - \lambda_Q \nabla_{\theta} L_Q(\theta)$
                \STATE $\hat{\theta} \leftarrow (1-\alpha)\hat{\theta} + \alpha\theta$
                \ENDFOR
                \STATE  \# Policy extraction (AWR or AlignIQL): 
                \FOR{each gradient step}
                \IF{AntMaze}
                \STATE \# Update policy with \Eqref{pi_soft}$+\kappa$\Eqref{AWR} \# AntMaze \footnotemark[1]
                \ELSE
                \STATE  Update policy with \Eqref{pi_soft} \# MuJoCo
                \ENDIF
                \ENDFOR
            \end{algorithmic}
\end{algorithm}

\footnotetext[1]{For AntMaze tasks, as discussed in Section~\ref{soft_solve}, achieving effective policy alignment requires strong policy expressivity. Due to the limited expressivity of Gaussian policies, Gaussian-based AlignIQL tends to overfit during early training. To mitigate this, we combine the weights of AWR and AlignIQL to enhance the multi-modality of extracted policy, i.e., $w(\vs,\va) = w_\text{AlignIQL} + \kappa w_\text{AWR}$. Moreover, the superior performance of diffusion-based AlignIQL using \Eqref{pi_soft} without any modification, compared to other diffusion-based baselines, further supports our hypothesis. An ablation study on $\kappa$ is provided in Appendix~\ref{extra_ablation}.}

The pseudocode for AlignIQL and AlignIQL-hard is provided in Algorithm~\ref{alg:training} and Algorithm~\ref{alg:extraction}, respectively. Note that AlignIQL shares the same training procedure as IQL; to implement AlignIQL, one only needs to replace the weight in IQL with the weight used in AlignIQL, as illustrated in Algorithm~\ref{alg:iql}. In fact, reimplementing AlignIQL from IQL is straightforward—only a single line corresponding to the policy extraction step needs to be modified, as shown below.

\clearpage
\newpage
\begin{python}
    def compute_actor_loss(
        self, batch: TorchMiniBatch, action: None
    ):
        # compute weight
        with torch.no_grad():
            v = self._modules.value_func(batch.observations)
            min_Q = self._targ_q_func_forwarder.compute_target(
                batch.observations, reduction="min"
            ).gather(1, batch.actions.long())
        # Weights for AlignIQL used in extracting the IQL policy    
        exp_a = torch.exp(((min_Q - v)**2) * self.eta).clamp(
            max=self._max_weight
        )
        # Weights for AWR used in extracting the IQL policy
        # exp_a = torch.exp((min_Q - v) * self._weight_temp).clamp(
        #    max=self._max_weight
        #)
        # compute log probability
        dist = self._modules.policy(batch.observations)
        log_probs = dist.log_prob(batch.actions.squeeze(-1)).unsqueeze(1)

        return ActorLoss(-(exp_a * log_probs).mean())
\end{python}

\section{Implementation Details and Additional Experiments}

\subsection{Implementation Details}
\label{details}
In this section, we introduce the implementation details for reproducing our results and some extra experiments to validate our method.

\textbf{Evaluation} Throughout this paper, unless otherwise specified, we report the final evaluation results averaged over different random seeds as our reported score.

\textbf{Gaussian-based implementation} Our Gaussian-based implementation is built upon CORL~\citep{tarasov2022corl}, an offline reinforcement learning library that offers high-quality, single-file implementations of state-of-the-art offline RL algorithms. Following AWR, we clip the weight in AlignIQL using $\max\cbr{0.01,\text{weight}}$.

\textbf{Diffusion-based implementation} Our Diffusion-based implementation is based on IDQL~\citep{hansen-estruch2023} and the jaxrl repo, which uses the JAX framework to implement RL algorithms. All networks are optimized through the Adam~\citep{kingma2014adam}. For D-AlignIQL-hard, we clip the multiplier network gradient to prevent gradient explosion due to the exponential term.  We use quantile loss and \Eqref{eq:Expectile_Weights} for IDQL since the expectile objective is used in IQL. For networks, we follow the default networks and parameters used by IDQL. The policy network uses an LN\_Resnet architecture~\cite{hansen-estruch2023} (Appendix G) with hidden size $256$ and $n=3$. The critic and value networks are 2-layer MLPs with a hidden size of $256$ and ReLU activation functions.  Following IDQL, we use normalization to adjust the rewards, which means $r=r/(r_\text{max}-r_\text{min})$. For AntMaze tasks, $r=r-1$. We also follow the IDQL's advice to take the maximum probability action at evaluation time. We train for $300000$ epochs for AntMaze tasks with batch size $512$ and $100000$ epochs for MuJoCo tasks with batch size $256$, consistent with IDQL and CORL.

\textbf{Data Corruption Details.} Following \citet{yang2023towards}, we apply random corruption to the four components: states, actions, rewards, and dynamics (i.e., next states). The overall corruption level is governed by four parameters: $c$,$\epsilon$, where $c$ denotes the corruption rate within an offline dataset of size $N$, while $\epsilon$ denotes the corruption scale for each dimension. Unless otherwise specified, we set $c=0.5$, and $\epsilon=0.5$. Below, we describe four types of random data corruption.
\begin{itemize}
     \item  \textbf{Random observation attack}: We randomly sample $c \cdot N$ transitions $(\vs, \va, r, \vs')$ and corrupt the states as $\hat \vs = \vs + \lambda \cdot \text{std}(\vs)$, where $\lambda \sim \text{Uniform}[-\epsilon, \epsilon]^{d_\vs}$. Here, $d_\vs$ denotes the dimensionality of the state, and $\text{std}(\vs)$ is the $d_\vs$-dimensional standard deviation of all states in the offline dataset. The noise is independently added to each dimension and scaled by the corresponding standard deviation. 
      \item  \textbf{Random action attack}: We randomly select $c \cdot N$ transitions $(\vs, \va, r, \vs')$ and corrupt the action as $\hat \va = \va + \lambda \cdot \text{std}(\va)$, where $\lambda \sim \text{Uniform}[-\epsilon, \epsilon]^{d_\va}$. Here, $d_\va$ denotes the dimensionality of the action, and $\text{std}(\va)$ is the $d_\va$-dimensional standard deviation of all actions in the offline dataset. 
    \item \textbf{Random reward attack}: We randomly sample $c \cdot N$ transitions $(\vs, \va, r, \vs')$ from $D$ and corrupt the reward as $\hat r \sim \text{Uniform}[-30 \cdot \epsilon, 30 \cdot \epsilon]$. 
    
    \item  \textbf{Random dynamics attack}: We randomly sample $c \cdot N$ transitions $(\vs, \va, r, \vs')$ and corrupt the next state as $\hat \vs' = \vs' + \lambda \cdot \text{std}(\vs')$, where $\lambda \sim \text{Uniform}[-\epsilon, \epsilon]^{d_\vs}$. Here, $d_\vs$ denotes the dimensionality of the state, and $\text{std}(\vs')$ is the $d_\vs$-dimensional standard deviation of all next states in the offline dataset.
        \item \textbf{Random mixed attack}: We randomly sample $c \cdot N$ transitions to conduct the random observation attack, followed by another $c \cdot N$ transitions for the random action attack. The same procedure is applied to the reward and dynamics attacks.
\end{itemize}

\textbf{Noise Data and Vision-based Experiment:} For noise data tasks, we train for $2e6$ steps on the D4RL halfcheetah-medium-replay-v2 robust tasks with $\epsilon=c=0.5$. Note that we use Gaussian-based AlignIQL (Algorithm~\ref{alg:iql}) in robust experiments  and image-based control, which means employing a Gaussian-based policy instead of the diffusion model.  We reimplement our method based on the official code from \citet{yang2023towards}. As shown in Appendix~\ref{pesudocode}, implementing our code based on IQL is very straightforward, requiring only changes to the policy extraction step. We use $\beta=\eta=3$ for both IQL+AWR (Abbreviated as IQL) and AlignIQL. For $\tau$, we adopt the default value $\tau=0.7$ provided in the official code from \citet{yang2023towards}. 

For vision-based experiments, we implement the discrete version of AlignIQL (Discrete AlignIQL) based on the discrete IQL (D-IQL) from d3rlpy. As shown in Appendix F.2, there is no price to implement AlignIQL based on IQL. For discrete IQL+AWR (Abbreviated as IQL), we report the average score of the last $3$ evaluations by selecting the minimal standard deviation from $\tau \in \sbr{0.5,0.7,0.9}$ in the last $3$ evaluations. Similarly, for Discrete AlignIQL with $\eta=1$, we report the average score of the last $3$ evaluations by selecting the minimal standard deviation from $\tau \in \sbr{0.5,0.7,0.9}$ in the last $3$ evaluations. We do this because vision-based methods are unstable, and their performance may fluctuate significantly across different seeds or training steps. 

% \textbf{Sparse Rewards Tasks:} For D-AlignIQL, we use $\eta=1>0$ and $\tau=0.7$ and keep other hyperparameters the same as Table~\ref{tab:hy1}. For D-AlignIQL-A, we sweep over $N\in\cbr{256,512,1024}$, and $\tau\in \cbr{0.7,0.9}$. For IDQL, we sweep over $\tau\in \cbr{0.7,0.9}$. The weight of IDQL is $w_2^{\tau}(s, a) = |\tau - \mathbbm{1}(Q(s,a) < V_\tau^2(s))|.$ We report the scores of Table~\ref{tab:sparse_results} by choosing the best score from different $N$. 

\textbf{Experimental details on different regularizers}: In this part, we aim to validate the effect of different regularizers. We experimented with the case of $f(x)=x-1$ in D-AlignIQL-hard and D-AlignIQL. Let $f(x)=x-1$, we can get $g_f(x)=\frac{1}{2}x+\frac{1}{2}$. Substituting back in \Eqref{pi_s_star} and \Eqref{pi_star} with $g_f(x)=\frac{1}{2}x+\frac{1}{2}$, we can get 
\begin{equation}
    \label{linear-soft}
        \text{AlignIQL:} \quad \pi^\star(\va\vert\vs) = \mu(\va\vert\vs)\max\cbr{\frac{1}{2}\pbr{-\alpha(\vs)-\eta\pbr{Q(\vs,\va)-V(\vs)}^2}+\frac{1}{2},0},
\end{equation}

\begin{equation}
    \label{linear-hard}
        \text{AlignIQL-hard:} \quad\pi^\star(\va\vert\vs) = \mu(\va\vert\vs)\max\cbr{\frac{1}{2}\pbr{-\alpha(\vs)-\beta(\vs)Q(\vs,\va)}+\frac{1}{2},0}.
\end{equation}
    We conducted experiments on Antmaze-umaze to evaluate the effects of different regularizers. We keep all other hyperparameters the same as Table~\ref{tab:hy1}. The experimental details are described as follows.

\textbf{D-AlignIQL:} In \Eqref{linear-soft}, $\alpha(\vs)$ serves as a normalization term, which does not affect the action evaluation when $\frac{1}{2}\pbr{-\alpha(\vs)-\eta\pbr{Q(\vs,\va)-V(\vs)}^2}+\frac{1}{2}\geq0$. To simply the training process, we assume $\frac{1}{2}\pbr{-\alpha(\vs)-\eta\pbr{Q(\vs,\va)-V(\vs)}^2}+\frac{1}{2}\geq0$ and ignore $\alpha(\vs)$. Since we use the Diffusion-based implementation and select the action with maximum weight, such simplification is reasonable and avoids training an extra multiplier network. We set $\eta=1$ in the Antmaze umaze experiment.

\textbf{AlignIQL-hard:} Similar to Lemma~\ref{loss_lemma}, we can train our multiplier through the following loss function (we replace $\frac{1}{2}\pbr{-\alpha(\vs)-\beta(\vs)Q(\vs,\va)}+\frac{1}{2}$ with $w_{\text{linear}}$ for simplicity)
     \begin{equation}
         \label{loss_linear_multiplier}
         \min_{\alpha, \beta} \gL_M=\E_{\va\sim\mu}\sbr{\mathbbm{1}\pbr{w_{\text{linear}}>0}w_{\text{linear}}^2}+\alpha(\vs)+\beta (\vs) V(\vs),
     \end{equation}
\begin{proof}
    This proof can be obtained by setting the gradient of \Eqref{loss_linear_multiplier} to $0$ with respect to $\alpha,\beta$.
\end{proof}

\subsection{Additional Experiments}
\label{add_exps}
% \paragraph{Results of Mix attacked} In this section, we employ random mix attack into the MuJoCo ``medium-replay'' datasets to further valid the motivation of policy alignment. Control corruption range $\epsilon=1$, we first randomly add noise into 10\% observations and then add noise into the 10\% rewards. The details of corruption are provided in Appendix~\ref{details}. We test the performance of different policy extraction methods from the 

% \begin{figure}
%     \centering
%     \includegraphics[width=\linewidth]{figures/environments_comparison_neurips.pdf}
%     \caption{Caption}
%     \label{fig:enter-label}
% \end{figure}

\paragraph{Results of AlignIQL-hard}

In this section, we report the results of D-AlignIQL-hard. Table~\ref{tab:hy1} reports the hyperparameters we used for AlignIQL-hard. According to our analysis, when $\beta < 0$, AlignIQL-hard is equivalent to using AWR for policy extraction. We only report the AntMaze results because the $\beta$ learned in MuJoCo tasks is essentially negative.
\begin{table*}[!htbp]
\caption{Average Results of D-AlignIQL-hard on AntMaze tasks.}
\label{table:align}
\centering
% \small
\begin{tabular}{l|ccc|ccc}
\hline
  & \multicolumn{3}{c|}{\bf D-AlignIQL-hard} & \multicolumn{3}{c}{\bf D-AlignIQL}\\ \hline
\multicolumn{1}{c|}{D4RL Tasks} & \multicolumn{1}{c}{$N=16$} & \multicolumn{1}{c}{$N=64$} & \multicolumn{1}{c|}{$N=256$} & \multicolumn{1}{c}{$N=16$} & \multicolumn{1}{c}{$N=64$} & \multicolumn{1}{c}{$N=256$} \\ \hline
% \midrule
% \multicolumn{1}{c|}{\bf Locomotion}  &      $69.7$       &  $71.0$   &  $71.5$ &  $73.2$ &  $74.0$ &  $72.3$    \\ \hline

% \midrule
\multicolumn{1}{c|}{\bf AntMaze} &      $54.2$       &  $57.9$ &    $56.7$ &  $65.8$ &  $70.2$ &  $70.7$    \\ \hline
% \bottomrule
\end{tabular}

\end{table*}

We also report the performance of D-AlignIQL under different $N$. Therefore, the results in Table~\ref{table:align} are slightly lower than those in Table~\ref{tbl:rl_results}. For D-AlignIQL and D-AlignIQL-hard,  we perform minimal hyperparameter tuning. In most cases, we use the default parameters of IDQL. Therefore, the performance of our algorithm can be further improved with additional tuning.

\paragraph{Running time}
\label{runtime}

The biggest problem of the diffusion-based method is the long inference time, which comes from the iterative running of the Markov chain. In this part, we present the running time of D-AlignIQL compared to other methods. We tested the runtime of DiffCPS on an RTX 3050 GPU on D4RL tasks. ($3000$ epochs (3e6 gradient steps)) From Table~\ref{table:time}, it's evident that the runtime of D-AlignIQL is comparable to other diffusion-based methods. 

\begin{table}[!htbp]
% \vspace{.1cm}
% \vspace{0.1in}
\caption{Runtime of different diffusion-based offline RL methods. (Average)}
\centering
\small
\scalebox{0.9}{
\begin{tabular}{lcccccccc}
\toprule
\multicolumn{1}{c}{\bf D4RL Tasks}  & \multicolumn{1}{c}{\bf D-AlignIQL (ours) (T=5) }
& \multicolumn{1}{c}{\bf DiffusionQL (T=5) }
& \multicolumn{1}{c}{\bf SfBC (T=5)} & \multicolumn{1}{c}{\bf IDQL (T=5)}   \\
\midrule
\multicolumn{1}{c}{\bf Locomotion Runtime ($1$ epoch)}  &      $9.12$s       &  $5.1$s   &  $8.4$s   &  $9.5$s  \\

\midrule
\multicolumn{1}{c}{\bf AntMaze Runtime ($1$ epoch)} &      $9.76$s       &  $10.5$s &    $10.5$s    &  $10.5$s \\

\bottomrule
\end{tabular}
}
\label{table:time}
% \vspace{-0.2in}
\end{table}
Although the runtime of D-AlignIQL is comparable to other diffusion-based methods, 
AlignIQL is still slower than the Gaussian-based policy (about $1.2$s for one epoch). The slow inference speed can harm the performance in real-time robot control tasks. Fortunately, this problem can be solved by recent sample acceleration methods, like SiD~\citep{zhou2024score,zhou2024long} or EDP~\citep{kang_efficient_nodate}. EDP directly constructs actions from corrupted ones at training to avoid running the sampling chain. In this way, EDP only needs to run the noise-prediction network once, which can substantially reduce the training time. Below, we first shortly introduce EDP

\textbf{EDP:} \citet{kang_efficient_nodate} noticed that the noisy sample of diffusion model can be written as $q(\mathbf{x}^t \vert \mathbf{x}^0) = \mathcal{N}(\mathbf{x}^t; \sqrt{\bar{\alpha}_t} \mathbf{x}^0, (1 - \bar{\alpha}_t)\mathbf{I})$.

Using the parametrization trick, we can get 
\begin{equation}
    \vx^t = \sqrt{\bar{\alpha}_t}\vx^0+\sqrt{1 - \bar{\alpha}_t}\epsilon, \quad \epsilon\in\gN(\mathbf{0},\mathbf{I})
\end{equation}

Replacing $\epsilon$ with our denoising network $\epsilon_\phi(\vx_i,i,\vs)$, we can obtain the action by running the noise-prediction once:
\begin{equation}
    \vx^0 = \frac{1}{\sqrt{\bar{\alpha}_t}}\vx^t-\frac{\sqrt{1 - \bar{\alpha}_t}}{\sqrt{\bar{\alpha}_t}}\epsilon_\phi(\vx_i,i,\vs)
\end{equation}
Although EDP is a simple method, it can greatly reduce the training time of diffusion-based offline RL methods while keeping competitive results. EDP can also enjoy the benefits of other diffusion acceleration methods, like DPM-solver~\cite{lu2022dpmsolverfastodesolver}.

We use the EDP's official IQL code to reimplement our method. 
   Table~\ref{table:distill} shows the results of EDP-based AlignIQL.
\begin{table}[ht]
    \centering
        \caption{Performance and runtime time ($1$ epoch) of D-AlignIQL (Diffusion steps $T=5$) and EDP-based D-AlignIQL.}

\begin{tabular}{c|c|c|c|c}
\hline
\multirow{2}{*}{Method}&
\multicolumn{2}{c|}{Performance} &
\multicolumn{2}{c}{Runtime (s)} \\
\cline{2-5}
  &Large-p &Large-d &Large-p &Large-d  \\
\hline
D-AlignIQL        &  $65.2$   &  $66.4$        & $9.5$ &   $9.78$  \\
\hline
EDP-based D-AlignIQL  & $43$ &    $62$  &  $2.22$ &  $1.95$\\
\hline
\end{tabular}
    \label{table:distill}
\end{table}

The above results are based on a single random seed, as our primary focus is on runtime efficiency. As shown in Table~\ref{table:distill}, a simple EDP-based AlignIQL implementation can reduce training time by up to 80\% while maintaining comparable performance to the original diffusion-based policy. Notably, our implementation does not utilize the DPM-solver, which, according to the original EDP paper, can further accelerate training by a factor of 2.3. In summary, diffusion-based policies with sample acceleration can achieve training speeds comparable to those of Gaussian policies (approximately 1.2 seconds per epoch).

\paragraph{Full D4RL results and Training Curves}  \label{full_res}
As shown in Table~\ref{tbl:rl_results} and Table~\ref{tbl:all_results}, the performance of both Gaussian-based AlignIQL and D-AlignIQL is inferior to IQL+AWR and other methods. As discussed in Section~\ref{aligniql-hard}, when $\beta(\vs) < 0$, AWR inherently achieves policy alignment. In MuJoCo tasks, $\beta$ is generally negative, so AWR alone suffices to achieve policy alignment, rendering AlignIQL unnecessary, as it provides a relaxed solution to policy alignment. However, for more challenging tasks such as AntMaze, $\beta$ is not always negative; in such cases, AlignIQL enables policy extraction without computing the implicit multiplier and approximately achieves policy alignment. Moreover, as demonstrated in our experiments, AlignIQL also improves the robustness of the extracted policy.

\begin{table*}[htbp]
\caption{The full results of Gaussian-based AlignIQL over $10$ random seeds. We use the Gaussian-based implementation of AlignIQL. The top $3$ results are highlighted in bold. The baseline results are taken from their original papers.}
\label{tbl:rl_results}
\centering
\small
\resizebox{1.0\textwidth}{!}{%
\begin{tabular}{llccccccccc}
\toprule
\multicolumn{1}{c}{\bf Dataset} & \multicolumn{1}{c}{\bf Environment} & \multicolumn{1}{c}{\bf CQL} & \multicolumn{1}{c}{\bf Diffusion-QL} & \multicolumn{1}{c}{\bf SfBC} & \multicolumn{1}{c}{\bf SQL} & \multicolumn{1}{c}{\bf DD} & \multicolumn{1}{c}{\bf Diffuser} & \multicolumn{1}{c}{\bf IDQL-A} & \multicolumn{1}{c}{\bf IQL} & \multicolumn{1}{c}{\bf AlignIQL (ours)}\\
\midrule
Medium-Expert & HalfCheetah    & $62.4$               &  \textbf{$\textbf{96.8}$} & $92.6$  & $\textbf{94.0}$   & $90.6$     & $79.8$       &  $\textbf{95.9}$ & $86.7$ & $81.9$\scriptsize{±1.50} \\
Medium-Expert & Hopper         & $98.7$               &  $\textbf{111.1}$    &  $108.6$ & \textbf{$\textbf{111.8}$}  & \textbf{$\textbf{111.8}$} & $107.2$ & $108.6$ & $91.5$ & $75.2$\scriptsize{±5.9} \\
Medium-Expert & Walker2d       & $\textbf{111.0}$              & $\textbf{110.1}$ & $109.8$ & $110.0$ & $108.8$ & $108.4$ & \textbf{$\textbf{112.7}$} & $109.6$ & $104.4$\scriptsize{±9.5} \\
\midrule
Medium        & HalfCheetah    & $44.4$               & \textbf{$\textbf{51.1}$}  & $45.9$  & $48.3$  & $\textbf{49.1}$ & $44.2$ & $\textbf{51.0}$ & $47.4$ & $44.2$\scriptsize{±0.3} \\
Medium        & Hopper         & $58.0$               & \textbf{$\textbf{90.5}$}  & $57.1$ & $\textbf{75.5}$  & $\textbf{79.3}$ & $58.5$ & $65.4$ & $66.3$ & $57.8$\scriptsize{±2.4} \\
Medium        & Walker2d       & $79.2$               & \textbf{$\textbf{87.0}$}  & $77.9$  & $\textbf{84.2}$  & $\textbf{82.5}$ & $79.7$ & $82.5$ & $78.3$ & $76.7$\scriptsize{±3.4} \\
\midrule
Medium-Replay & HalfCheetah    & $\textbf{46.2}$               & \textbf{$\textbf{47.8}$}  & $37.1$  & $44.8$  & $39.3$ & $42.2$ & $\textbf{45.9}$ & $44.2$ & $37.3$\scriptsize{±0.2} \\
Medium-Replay & Hopper         & $48.6$               & \textbf{$\textbf{101.3}$} & $86.2$  & $\textbf{99.7}$  & $\textbf{100.0}$ & $96.8$ & $92.1$ & $94.7$ & $77.9$\scriptsize{±8.9} \\
Medium-Replay & Walker2d       & $26.7$               & \textbf{$\textbf{95.5}$} & $65.1$  & $\textbf{81.2}$  & $75.0$ & $61.2$ & $\textbf{85.1}$ & $73.9$ & $66.3$\scriptsize{±9.1} \\
\midrule
\multicolumn{2}{c}{\bf Average (Locomotion)} & $63.9$ & \textbf{$\textbf{87.9}$} & $75.6$ & $83.3$  & $81.8$ & $75.3$ & $82.1$ & $76.9$ & $68.1$ \\

    \midrule
    Default       & AntMaze-umaze  & $74.0$               & $\textbf{93.4}$ & $92.0$  & $92.2$  & -       & -      & $\textbf{94.0}$ & $87.5$ & $\textbf{95.6}$\scriptsize{±2.2} \\
    Diverse       & AntMaze-umaze  & $\textbf{84.0}$          & $66.2$  & $\textbf{85.3}$  & $74.0$ & -       & -      & $80.2$ & $62.2$ & $\textbf{72.0}$\scriptsize{±7.3} \\
    \midrule
    Play          & AntMaze-medium & $61.2$               & $76.6$  & $\textbf{81.3}$  & $80.2$  & -       & -      & $\textbf{84.5}$ & $71.2$ & $\textbf{88.0}$\scriptsize{±2.7}\\
    Diverse       & AntMaze-medium & $53.7$               & $78.6$  & $\textbf{82.0}$  & $79.1$  & -       & -      & $\textbf{84.8}$ & $70.0$ & $\textbf{83.2}$\scriptsize{±5.2} \\
    \midrule
    Play          & AntMaze-large  & $15.8$               & $46.4$  & $\textbf{59.3}$  & $53.2$  & -       & -      & $\textbf{63.5}$ & $39.6$ & $\textbf{55.2}$\scriptsize{±9.5} \\
    Diverse       & AntMaze-large  & $14.9$               & $\textbf{57.3}$  & $45.5$  & $52.3$  & -       & -      & $\textbf{67.9}$ & $47.5$ & $\textbf{58.0}$\scriptsize{±3.6} \\
    \midrule
    \multicolumn{2}{c}{\bf Average (AntMaze)} & $50.6$ & $69.8$ & $74.2$ & $-$ & - & - & $$79.1$$ & $63.0$ & $$\textbf{75.3}$$ \\
    \midrule
    \multicolumn{2}{c}{\bf{\# Diffusion steps}} & - & $5$ & $15$ & $-$ & $100$ & $100$ & $5$ & - & $-$ \\
\bottomrule
\end{tabular}
}

\end{table*}

\begin{table*}[!htbp]
    \caption{The full results of D-AlignIQL over $10$ random seeds. The top $3$ results in each D4RL task and the best average results are highlighted in bold. The baseline results are taken from their original papers.}
    \label{tbl:all_results}
    \vspace{0.2cm}
    \centering
    \small
    \resizebox{1.0\textwidth}{!}{%
    \begin{tabular}{llccccccccc}
    \toprule
    \multicolumn{1}{c}{\bf Dataset} & \multicolumn{1}{c}{\bf Environment} & \multicolumn{1}{c}{\bf CQL} & \multicolumn{1}{c}{\bf Diffusion-QL} & \multicolumn{1}{c}{\bf SfBC} & \multicolumn{1}{c}{\bf SQL} & \multicolumn{1}{c}{\bf DD} & \multicolumn{1}{c}{\bf Diffuser} & \multicolumn{1}{c}{\bf IDQL} & \multicolumn{1}{c}{\bf IQL} & \multicolumn{1}{c}{\bf D-AlignIQL (ours)}\\
    \midrule
    Medium-Expert & HalfCheetah    & $62.4$               &  $\textbf{96.8}$ & $92.6$  & $\textbf{94.0}$   & $90.6$     & $79.8$       &  $\textbf{95.9}$ & $86.7$ & $89.1$\scriptsize{±0.6} \\
    Medium-Expert & Hopper         & $98.7$               &  $\textbf{111.1}$    &  $108.6$ & $\textbf{111.8}$  & $\textbf{111.8}$ & $107.2$ & $108.6$ & $91.5$ & $107.1$\scriptsize{±0.2} \\
    Medium-Expert & Walker2d       & $\textbf{111.0}$              & $110.1$ & $109.8$ & $110.0$ & $108.8$ & $108.4$ & $\textbf{112.7}$ & $109.6$ & $\textbf{111.9}$\scriptsize{±0.8} \\
    \midrule
    Medium        & HalfCheetah    & $44.4$               & $\textbf{51.1}$  & $45.9$  & $48.3$  & $\textbf{49.1}$ & $44.2$ & $\textbf{51.0}$ & $47.4$ & $46.0$\scriptsize{±4.6} \\
    Medium        & Hopper         & $58.0$               & $\textbf{90.5}$  & $57.1$ & $\textbf{75.5}$  & $\textbf{79.3}$ & $58.5$ & $65.4$ & $66.3$ & $60.5$\scriptsize{±0.5} \\
    Medium        & Walker2d       & $79.2$               & $\textbf{87.0}$  & $77.9$  & $\textbf{84.2}$  & $\textbf{82.5}$ & $79.7$ & $82.5$ & $78.3$ & $79.2$\scriptsize{±2.7} \\
    \midrule
    Medium-Replay & HalfCheetah    & $\textbf{46.2}$               & $\textbf{47.8}$  & $37.1$  & $44.8$  & $39.3$ & $42.2$ & $\textbf{45.9}$ & $44.2$ & $41.1$\scriptsize{±2.8} \\
    Medium-Replay & Hopper         & $48.6$               & $\textbf{101.3}$ & $86.2$  & $\textbf{99.7}$  & $\textbf{100.0}$ & $96.8$ & $92.1$ & $94.7$ & $56.2$\scriptsize{±3.5} \\
    Medium-Replay & Walker2d       & $26.7$               & $\textbf{95.5}$ & $65.1$  & $\textbf{81.2}$  & $75.0$ & $61.2$ & $\textbf{85.1}$ & $73.9$ & $58.7$\scriptsize{±0.6} \\
    \midrule
    \multicolumn{2}{c}{\bf Average (Locomotion)} & $63.9$ & $\textbf{87.9}$ & $75.6$ & $$83.3$$ & $81.8$ & $75.3$ & $82.1$ & $76.9$ & $72.2$ \\
    \midrule
    Default       & AntMaze-umaze  & $74.0$               & $\textbf{93.4}$ & $92.0$  & $92.2$  & -       & -      & $\textbf{94.0}$ & $87.5$ & $\textbf{94.8}$\scriptsize{±3.2} \\
    Diverse       & AntMaze-umaze  & $\textbf{84.0}$          & $66.2$  & $\textbf{85.3}$  & $74.0$ & -       & -      & $80.2$ & $62.2$ & $\textbf{82.4}$\scriptsize{±4.4} \\
    \midrule
    Play          & AntMaze-medium & $61.2$               & $76.6$  & $\textbf{81.3}$  & $80.2$  & -       & -      & $\textbf{84.5}$ & $71.2$ & $\textbf{87.5}$\scriptsize{±2.5} \\
    Diverse       & AntMaze-medium & $53.7$               & $78.6$  & $\textbf{82.0}$  & $79.1$  & -       & -      & $\textbf{84.8}$ & $70.0$ & $\textbf{85.0}$\scriptsize{±5.0} \\
    \midrule
    Play          & AntMaze-large  & $15.8$               & $46.4$  & $\textbf{59.3}$  & $53.2$  & -       & -      & $\textbf{63.5}$ & $39.6$ & $\textbf{65.2}$\scriptsize{±9.6} \\
    Diverse       & AntMaze-large  & $14.9$               & $\textbf{57.3}$  & $45.5$  & $52.3$  & -       & -      & $\textbf{67.9}$ & $47.5$ & $\textbf{66.4}$\scriptsize{±9.7} \\
    \midrule
    \multicolumn{2}{c}{\bf Average (AntMaze)} & $50.6$ & $69.8$ & $74.2$ & $-$ & - & - & $$79.1$$ & $63.0$ & $$\textbf{80.2}$$ \\
    \midrule
    \multicolumn{2}{c}{\bf{\# Diffusion steps}} & - & $5$ & $15$ & $-$ & $100$ & $100$ & $5$ & - & $5$ \\
    \bottomrule
    \end{tabular}
    }

    \end{table*}

\begin{table}[ht]
\centering
\caption{Performance of Gaussian-based AlignIQL under different $\kappa$ on AntMaze tasks. AlignIQL($\kappa>0$) outperforms IQL+AWR on all AntMaze tasks, and the difference between $\kappa=0.01$ and $\kappa=0.1$ is small (around $5\%$).}
\resizebox{1.0\textwidth}{!}{%
\begin{tabular}{lcccccc|c}
\toprule
Method & Umaze & Umaze-Diverse & Medium-Play & Medium-Diverse & Large-Play & Large-Diverse & Average \\
\midrule
IQL & 87.5 & 62.2 & 71.2 & 70.0 & 39.6 & 47.5 & 63.0 \\
AlignIQL ($\kappa=0.0$)     & 87.2 ± 4.8 & 68.8 ± 7.7 & 20.4 ± 7.9 & 34.8 ± 14.1 & 4.8±4.6 & 4.0±1.8 & 36.7 \\
AlignIQL ($\kappa=0.01$) & 90.8 ± 1.1 & 73.2 ± 5.8 & 74.0 ± 10.9 & 82.0± 3.6 & 34.4 ± 14.0 & 26.0 ± 3.7 & 63.4\\
AlignIQL ($\kappa=0.1$) & 95.6 ± 2.2 & 66.0 ± 10.5 & 74.8 ± 9.4 & 81.2 ± 4.6 & 53.6 ± 6.7 & 48.8 ± 8.1 & 70.0 \\
\bottomrule
\end{tabular} }
\label{table:ant_policy_align}
\end{table}

\paragraph{Extra Ablation Study} 
\label{extra_ablation}
In this section, we put the full ablation study results of $\eta$  or $\kappa$ for Gaussian-based AlignIQL in D4RL tasks.  Tables~\ref {table:full_ablation2} and~\ref {table:ant_policy_align} report the corresponding final evaluation results over $10$ random seeds. 

\begin{table}[htbp]
    \centering
        \caption{Performance  of AlignIQL under different $\eta$ over $10$ random seeds. }
\begin{tabular}{c|c|c|c|c|c|c|c|c|c|c}
\hline
\multirow{2}{*}{$\eta$}&
\multicolumn{3}{c|}{Walker2d} &
\multicolumn{3}{c|}{Halfcheetah} &
\multicolumn{3}{c|}{Hopper}&\multirow{2}{*}{Average}\\
\cline{2-10} 
 & M&ME &MR &M&ME &MR &M&ME &MR \\
\hline
$\eta=0.5$      & $74.4$  &  $87.2$   &  $28.3$     & $43.1$   & $79.1$ &   $38.8$  & $55.2$ & $35.5$ & $76.0$&  $57.5$\\
\hline
$\eta=1$    & $72.8$    &  $84.3$   &  $51.7$  & $43.5$       & $64.3$ &   $36.0$ & $55.3$ & $41.1$ & $72.8$&  $58.0$\\
\hline
$\eta=3$   & $75.3$     &  $88.4$   &  $62.9$   & $43.8$      & $68.7$ &   $38.1$  
 & $56.2$ & $52.2$ & $79.3$&  $62.8$\\
\hline
$\eta=5$ & $76.2$ & $86.2$ &    $63.9$ & $43.9$   &  $67.4$ &  $39.6$ & $57.1$ & $94.9$ & $80.1$&  $67.7$\\
\hline 
$\eta=10$ & $76.7$ & $104.4$ &    $66.3$ & $44.2$  &  $73.3$ &  $37.3$ & $57.8$ & $75.2$ & $77.9$&  $68.1$\\ 
\hline
\end{tabular}
    \label{table:full_ablation2}
\end{table}

\clearpage
\newpage
\section{Hyperparameters}
 We provide the main hyperparameters in Table~\ref{tab:hy1} to reproduce our results in Table~\ref{tbl:rl_results} and Table~\ref{tbl:all_results}. Here are the hyperparameters for reproducing our results.

\begin{center}
\label{tab:hy1}
\def\arraystretch{1.0}
\begin{tabular}{|l|c|} 
\hline
\textbf{LR} (For all networks except for multiplier)  & 3e-4 \\
\textbf{LR} (Multiplier Network)  & 3e-5 \\
\textbf{Critic Batch Size}  & 512 \\
\textbf{Actor Batch Size}  & 512 \\
\textbf{$\tau$ Expectiles}  & $0.7$ (locomotion), $0.9$ (AntMaze)\\
\multirow{2}{*}{\textbf{$\eta$ For AlignIQL and D-AlignIQL}} & $10$ (MuJoCo)\\ &$1$ (D-AlignIQL)\\
\textbf{$\kappa$ for AlignIQL on AntMaze Tasks}  & $\kappa=0.1$ \\
\textbf{$\eta$ For AlignIQL on AntMaze Tasks}  & $3$ (Umaze),$10$ (Umaze-d),$300$ (M-P,M-d,L-d,L-p)  \\
\multirow{2}{*}{\textbf{Grad norm for multiplier on MuJoCo in AlignIQL-hard}}  & $1.0$ ($\alpha$)\\& $0.5$ ($\beta$)\\
\multirow{2}{*}{\textbf{Grad norm for multiplier on AntMaze in AlignIQL-hard}}  & $1.0$ ($\alpha$)\\& $1$ ($\beta$)\\
\textbf{Critic Grad Steps} & 3e6 \\
\textbf{Actor Grad Steps}  & 3e6 \\
\textbf{Target Critic EMA} & 0.005 \\
\textbf{T} & 5 \\
% \textbf{N} & 64 (AntMaze), 256 (MuJoCO) \\
\textbf{Beta schedule} & \text{Variance Preserving~\citep{song2020score} } \\
\textbf{Actor Dropout Rate} & 0.1 for actor on all tasks \\
\textbf{Critic Dropout Rate} & 0.1 for AntMaze Tasks in AlignIQL-hard \\
\textbf{Number Residual Blocks} & 3 \\
\textbf{Actor Cosine Decay~\citep{loshchilov2016sgdr}}  & Number of Actor Grad Steps\\
\textbf{Optimizer} & Adam \citep{kingma2014adam}\\
\multirow{1}{*}{\textbf{$N$ For D-AlignIQL}} & $256$ (MuJoCo,AntMaze)\\
\hline
\end{tabular}
\end{center}
\begin{figure*}[htbp]
\centering
% \vspace{-0.1in}
\includegraphics[width = \linewidth]{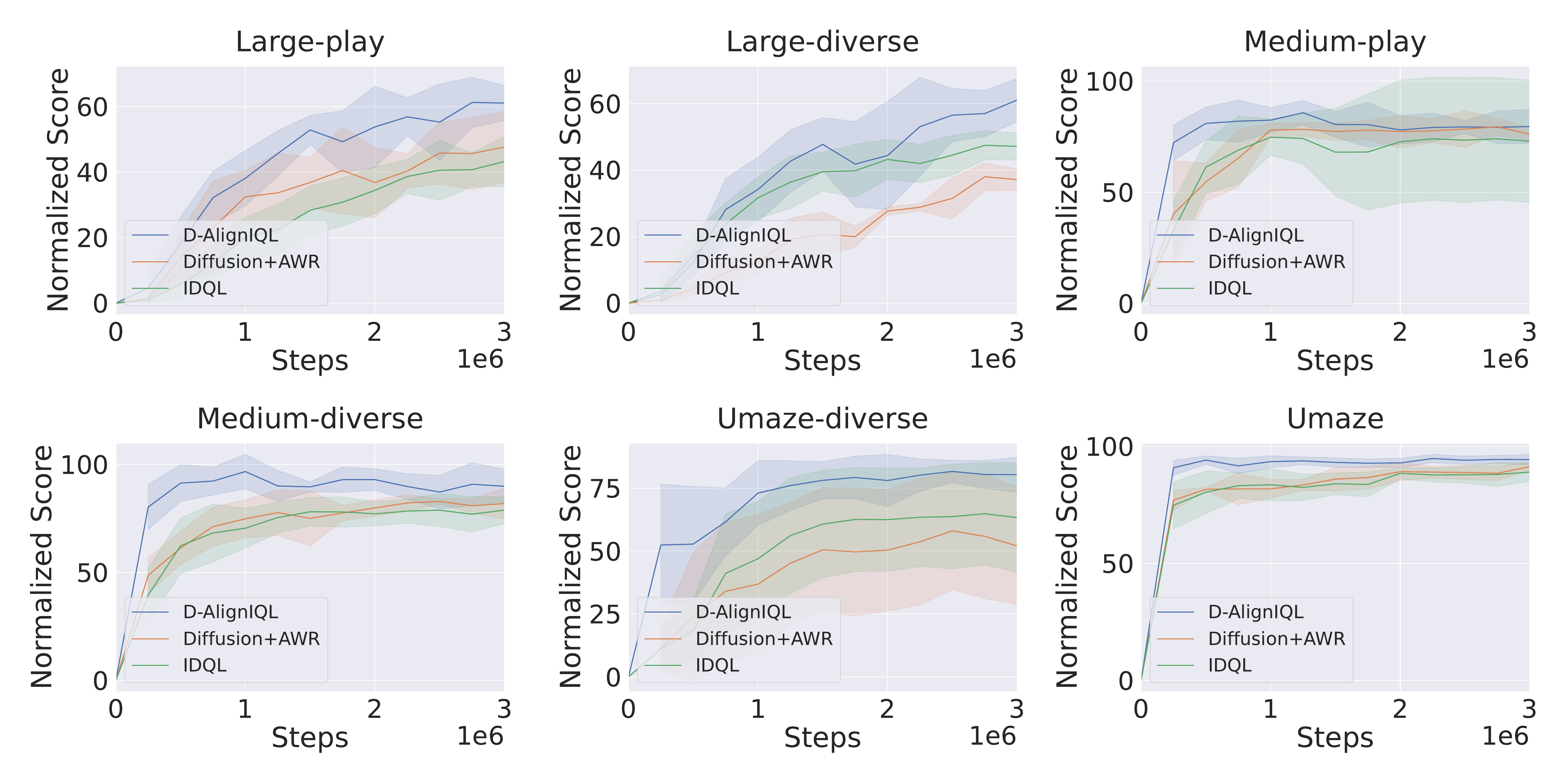}
% \vspace{-0.2in}
\caption{Training curves of D-AlignIQL, IDQL, and Diffusion+AWR. The normalized score is calculated by averaging the scores across three different $N$ ($N=16,64,256$).}
\label{fig:training curve}
\end{figure*}

\clearpage
\newpage
\section*{NeurIPS Paper Checklist}

\begin{enumerate}

\item {\bf Claims}
    \item[] Question: Do the main claims made in the abstract and introduction accurately reflect the paper's contributions and scope?
    \item[] Answer: \answerYes{} % Replace by \answerYes{}, \answerNo{}, or \answerNA{}.
    \item[] Justification: Section Introduction

\item {\bf Limitations}
    \item[] Question: Does the paper discuss the limitations of the work performed by the authors?
    \item[] Answer: \answerYes{} % Replace by \answerYes{}, \answerNo{}, or \answerNA{}.
    \item[] Justification: Section Conclusion

\item {\bf Theory assumptions and proofs}
    \item[] Question: For each theoretical result, does the paper provide the full set of assumptions and a complete (and correct) proof?
    \item[] Answer: \answerYes{} % Replace by \answerYes{}, \answerNo{}, or \answerNA{}.
    \item[] Justification: Section Appendix B

    \item {\bf Experimental result reproducibility}
    \item[] Question: Does the paper fully disclose all the information needed to reproduce the main experimental results of the paper to the extent that it affects the main claims and/or conclusions of the paper (regardless of whether the code and data are provided or not)?
    \item[] Answer: \answerYes{} % Replace by \answerYes{}, \answerNo{}, or \answerNA{}.
    \item[] Justification: Section Appendix D

\item {\bf Open access to data and code}
    \item[] Question: Does the paper provide open access to the data and code, with sufficient instructions to faithfully reproduce the main experimental results, as described in supplemental material?
    \item[] Answer: \answerYes{} % Replace by \answerYes{}, \answerNo{}, or \answerNA{}.
    \item[] Justification: See Attachments

\item {\bf Experimental setting/details}
    \item[] Question: Does the paper specify all the training and test details (e.g., data splits, hyperparameters, how they were chosen, type of optimizer, etc.) necessary to understand the results?
    \item[] Answer: \answerYes{} % Replace by \answerYes{}, \answerNo{}, or \answerNA{}.
    \item[] Justification: Section Appendix D

\item {\bf Experiment statistical significance}
    \item[] Question: Does the paper report error bars suitably and correctly defined or other appropriate information about the statistical significance of the experiments?
    \item[] Answer: \answerYes{} % Replace by \answerYes{}, \answerNo{}, or \answerNA{}.
    \item[] Justification: Section Experiments

\item {\bf Experiments compute resources}
    \item[] Question: For each experiment, does the paper provide sufficient information on the computer resources (type of compute workers, memory, time of execution) needed to reproduce the experiments?
    \item[] Answer: \answerYes{} % Replace by \answerYes{}, \answerNo{}, or \answerNA{}.
    \item[] Justification: Section Experiments
    
\item {\bf Code of ethics}
    \item[] Question: Does the research conducted in the paper conform, in every respect, with the NeurIPS Code of Ethics \url{https://neurips.cc/public/EthicsGuidelines}?
    \item[] Answer: \answerYes{} % Replace by \answerYes{}, \answerNo{}, or \answerNA{}.
    \item[] Justification: Our paper does not involve these issues.

\item {\bf Broader impacts}
    \item[] Question: Does the paper discuss both potential positive societal impacts and negative societal impacts of the work performed?
    \item[] Answer: \answerYes{} % Replace by \answerYes{}, \answerNo{}, or \answerNA{}.
    \item[] Justification: Appendix 
    
\item {\bf Safeguards}
    \item[] Question: Does the paper describe safeguards that have been put in place for responsible release of data or models that have a high risk for misuse (e.g., pretrained language models, image generators, or scraped datasets)?
    \item[] Answer: \answerNo{} % Replace by \answerYes{}, \answerNo{}, or \answerNA{}.
    \item[] Justification: Our paper does not involve these issues.

\item {\bf Licenses for existing assets}
    \item[] Question: Are the creators or original owners of assets (e.g., code, data, models), used in the paper, properly credited and are the license and terms of use explicitly mentioned and properly respected?
    \item[] Answer: \answerYes{} % Replace by \answerYes{}, \answerNo{}, or \answerNA{}.
    \item[] Justification: Section Experiments

\item {\bf New assets}
    \item[] Question: Are new assets introduced in the paper well documented and is the documentation provided alongside the assets?
    \item[] Answer: \answerYes{} % Replace by \answerYes{}, \answerNo{}, or \answerNA{}.
    \item[] Justification: See Attachments

\item {\bf Crowdsourcing and research with human subjects}
    \item[] Question: For crowdsourcing experiments and research with human subjects, does the paper include the full text of instructions given to participants and screenshots, if applicable, as well as details about compensation (if any)? 
    \item[] Answer: \answerNo{} % Replace by \answerYes{}, \answerNo{}, or \answerNA{}.
    \item[] Justification: Our paper does not involve these issues.

\item {\bf Institutional review board (IRB) approvals or equivalent for research with human subjects}
    \item[] Question: Does the paper describe potential risks incurred by study participants, whether such risks were disclosed to the subjects, and whether Institutional Review Board (IRB) approvals (or an equivalent approval/review based on the requirements of your country or institution) were obtained?
    \item[] Answer: \answerNo{} % Replace by \answerYes{}, \answerNo{}, or \answerNA{}.
    \item[] Justification: Our paper does not involve these issues.

\item {\bf Declaration of LLM usage}
    \item[] Question: Does the paper describe the usage of LLMs if it is an important, original, or non-standard component of the core methods in this research? Note that if the LLM is used only for writing, editing, or formatting purposes and does not impact the core methodology, scientific rigorousness, or originality of the research, declaration is not required.
    %this research? 
    \item[] Answer: \answerNA{}% Replace by \answerYes{}, \answerNo{}, or \answerNA{}.
    \item[] Justification: The core method development in this research does not involve LLMs as any important, original, or non-standard components.

\end{enumerate}

\end{document}